\documentclass[11pt]{article}
\usepackage[ngerman,english]{babel}
\usepackage{amsmath, amsxtra, amsfonts, amssymb, amstext}
\usepackage{amsthm}
\usepackage{booktabs}
\usepackage{fullpage}
\usepackage{nicefrac}
\usepackage{xspace}
\usepackage[noadjust]{cite}
\usepackage{url}
\usepackage{graphics}
\usepackage[usenames,dvipsnames]{xcolor}
\usepackage[colorlinks]{hyperref}
\definecolor{linkblue}{rgb}{0.1,0.1,0.8}
\hypersetup{colorlinks=true,linkcolor=linkblue,filecolor=linkblue,urlcolor=linkblue,citecolor=linkblue}
\usepackage[algo2e,ruled,vlined,linesnumbered]{algorithm2e}
\usepackage{wrapfig}
\usepackage{pdfpages}
\usepackage{fullpage}
\usepackage{multirow}
\usepackage{rotating}
\usepackage{colortbl}

\usepackage{ifthen}
\newcommand{\IsChapter}{false}

\newtheorem{theorem}{Theorem}

\newcommand{\N}{\mathbb{N}}
\newcommand{\R}{\mathbb{R}}

\renewcommand{\epsilon}{\varepsilon}
\newcommand{\eps}{\epsilon}

\DeclareMathOperator{\E}{E}

\DeclareMathOperator{\opt}{opt}
\DeclareMathOperator{\nonopt}{nonopt}
\DeclareMathOperator{\mut}{flip}
\DeclareMathOperator{\cross}{cross}
\newcommand{\plateau}{\textsc{Plateau}\xspace}

\DeclareMathOperator{\Bin}{Bin}

\newcommand{\assign}{\leftarrow}

\newcommand{\oea}{$(1 + 1)$~EA\xspace}
\newcommand{\lea}{$(1 + \lambda)$~EA\xspace}
\newcommand{\opl}{$(1 + \lambda)$~EA\xspace}
\newcommand{\opladap}{$(1 + \lambda)$~EA$_{(2r,r/2)}$\xspace}
\newcommand{\ocl}{$(1 , \lambda)$~EA\xspace}

\newcommand{\ga}{$(1 + (\lambda,\lambda))$~GA\xspace}

\newcommand{\OneMax}{\textsc{OneMax}\xspace}
\newcommand{\onemax}{\textsc{OneMax}\xspace}
\newcommand{\OM}{\textsc{Om}\xspace}

\newcommand{\leadingones}{\textsc{LeadingOnes}\xspace}

\newcommand{\LO}{\textsc{Lo}\xspace}

\newcommand{\opllga}{$(1 + (\lambda,\lambda))$~GA\xspace}

\DeclareMathOperator{\RLS}{RLS}

\begin{document}
    
\title{Theory of Parameter Control for Discrete Black-Box Optimization:\\ Provable Performance Gains Through Dynamic Parameter Choices}

\author{Benjamin Doerr$^{1}$,  
Carola Doerr$^{2}$}
 
\date{
$^1$\'Ecole Polytechnique, CNRS, Laboratoire d'Informatique (LIX), Palaiseau, France\\
$^2$ Sorbonne Universit\'e, CNRS, Laboratoire d'Informatique de Paris 6, LIP6, 75005 Paris, France
\\[1cm]
\today
}

\sloppypar
\maketitle
 
\begin{abstract}
Parameter control aims at realizing performance gains through a dynamic choice of the parameters which determine the behavior of the underlying optimization algorithm. In the context of evolutionary algorithms this research line has for a long time been dominated by empirical approaches. With the significant advances in running time analysis achieved in the last ten years, the parameter control question has become accessible to theoretical investigations. A number of running time results for a broad range of different parameter control mechanisms have been obtained in recent years. This \ifthenelse{\equal{\IsChapter}{true}}{chapter}{book chapter} surveys these works, and puts them into context, by proposing an updated classification scheme for parameter control.

\vspace{2ex}
\ifthenelse{\equal{\IsChapter}{true}}{}{Author-generated version of a book chapter that has appeared in \emph{Benjamin Doerr and Frank Neumann (Editors). Theory of Evolutionary Computation---Recent Developments in Discrete Optimization. Springer, 2020. \url{https://doi.org/10.1007/978-3-030-29414-4}}.}
\end{abstract}
%

\newpage
\tableofcontents

\newpage
\sloppy{
\section{Introduction}\label{sec:SAintro}

Evolutionary algorithms and many other iterative black-box optimization heuristics are parametrized algorithms; i.e., their search behavior depends (to a large extent) on a set of parameters which the user needs to specify, or which are set by the algorithm designer to some default values. It is today well understood that the parameter choice can have a very decisive influence on the performance of the heuristic~\cite{ParamSettBook}. Understanding how to best choose the parameters is therefore an important task. It is referred to as the \emph{parameter selection problem.}

The parameter setting problem is difficult for several reasons. 
\begin{itemize}
	\item \textbf{Complexity of Performance Prediction.} Despite significant research efforts devoted to this problem, predicting how the performance of an algorithm depends on the chosen parameter values remains a very challenging problem---both with empirical and theoretical methods. In fact, determining optimal parameter values can be very complex already for a single parameter. Many black-box optimization heuristics, however, rely on two or more parameters. Rigorously analyzing the interdependency between these parameters is often infeasible by state-of-the-art technology. 
	\item \textbf{Problem- and Instance-Dependence.} It is well known that no globally good parameter values exist, but that suitable parameter values can differ substantially between different optimization problems, and even between different instances of the same problem.  
	\item \textbf{State-Dependence.} It is furthermore widely acknowledged that the best parameter values can change during the optimization process. For example, it is often beneficial to use larger mutation rates in the beginning of an optimization process, to allow for a faster exploration, and to shrink the search radius over time, to allow for a better exploitation in the later stages, cf. Section~\ref{sec:SAmotivation} for a detailed example.
\end{itemize}

To overcome these difficulties, a large number of different parameter setting techniques have been developed. Following standard terminology in evolutionary computation, they can be classified into two main approaches:
\begin{itemize}
	\item \textbf{Static Parameter Settings: Parameter Tuning.} Parameter tuning aims at identifying parameter values that are, for a given algorithm on a given problem (instance), globally suitable throughout the whole optimization process. The parameters are initialized with these values and do not change during the optimization process. Parameter tuning thus addresses the above-mentioned problem- and instance-dependence of optimal parameter choices, but not their state-dependence.
	
	In \emph{empirical works,} parameter tuning often requires an initial set of experiments that support an informed decision. Automated tools that help the user to identify reasonable static parameter values are available, and have shown to bring significant performance gains over a manual tuning process. \cite{irace,SPOT,GGA,ParamILS,SMAC} are examples for automated parameter tuning approaches that have been used in (and to some extend specifically designed for) evolutionary optimization contexts.
	
	In \emph{theoretical works,} parameter tuning requires running time bounds that depend on the parameters under investigation. The minimization of these performance bounds then suggests suitable parameter values. A prime \emph{example} for such a mathematical approach towards parameter tuning is the precise running time bound for the \oea with mutation rate $p=c/n$ on linear functions. Witt~\cite{Witt13j} has shown that this expected optimization time is $(1\pm o(1)) \tfrac{e^c}{c} n \ln (n)$. This bound, together with larger running time bounds for mutation rates $p \neq c/n$, proves that the often recommended choice $p=1/n$ is indeed optimal for the \oea on this problem. Such precise upper and lower bounds, however, are rare. Even worse, only few running time bounds that depend on two or more parameters exist, cf. Section~\ref{sec:SAfitnessTheoretical}.
	\item \textbf{Dynamic Parameter Setting: Parameter Control.} Parameter control, in contrast, aims to benefit from a non-static choice of the parameters, with the underlying idea that the flexibility in the behavior can be used to adjust the algorithms' behavior to the current state of the optimization process. Put differently, parameter control does not only aim at \emph{identifying} parameter values that are a good compromise for the whole optimization process, but aims also at \emph{tracking} the evolution of the best parameter values. Even when the optimal parameter values are rather stable, the role of parameter control is to identify these values \emph{on the fly}, without a dedicated tuning step that precedes the actual optimization process. 
\end{itemize}

 This book chapter focuses on non-static parameter choices, and thus on parameter control mechanisms. We survey existing theoretical works of parameter control in the context of evolutionary algorithms and other standard black-box optimization heuristics. We also summarize a few standard techniques used in the empirical research literature.\footnote{Readers interested in empirical works on parameter control are referred to~\cite{KarafotiasHE15} for an exhaustive survey. Additional pointers can be found in the systematic literature review~\cite{AletiM16}, the book chapter~\cite{EibenMSS07} (and other book chapters in the same collection) and the seminal paper~\cite{EibenHM99}.} We structure our presentation by a new classification scheme for parameter control mechanisms. This taxonomy builds on the well-known classification by Eiben, Hinterding, and Michalewicz~\cite{EibenHM99}, but modifies it to better reflect the developments that parameter control has witnessed in the last 20 years. 

This \ifthenelse{\equal{\IsChapter}{true}}{chapter}{book chapter} is structured as follows. We motivate the use of non-static parameter choices in Section~\ref{sec:SAmotivation} by demonstrating a simple example where adaptive parameter selection is provably beneficial. We then introduce our revised classification scheme in Section~\ref{sec:SAclassification}. In the subsequent Sections~\ref{sec:SAstate} to~\ref{sec:SAhyper} we survey existing theoretical results. In Section~\ref{sec:SAconclusions} we conclude this book chapter with a discussion of promising avenues for future work. A summary of selected theoretical running time results covered in this book chapter can be found in Table~\ref{tab:SAsummaryResults}.

\section{A Motivating Example: (1+1)~EA and RLS on LeadingOnes}\label{sec:SAmotivation}

We start this section with an example that demonstrates potential advantages of parameter control mechanisms. To this end, we study the well-known \leadingones benchmark, the problem of minimizing an unknown function of the type 
$$\LO_{z,\sigma}:\{0,1\}^n \rightarrow [0..n]:=\{0,1,\ldots,n\}, x \mapsto \max \{ i \in [0..n] \mid \forall j \in [1..i]: x_{\sigma(j)} = z_{\sigma(j)} \},$$
where $z \in \{0,1\}^n$ and $\sigma$ is a permutation (one-to-one map) of the set $[1..n]$. Optimizing $\LO_{z,\sigma}$ corresponds to identifying $z$, the unique optimum of $\LO_{z,\sigma}$. Note that for every $x$ the function value $\LO_{z,\sigma}(x)$ is the length of the longest prefix that $x$ and $z$ have in common, when comparing the strings in the order prescribed by $\sigma$. 

It has been shown in~\cite{BottcherDN10} that the \oea with static mutation rate $0<p<1$ needs $\tfrac{1}{2p^2}[(1-p)^{1-n} - (1-p)]$ iterations, on average, to optimize a \leadingones instance~\cite{BottcherDN10}. This term is minimized for $p\approx 1.59/n$, which yields an expected optimization time of around $0.77 n^2$. It was observed in~\cite{BottcherDN10} that a fitness-dependent choice of the mutation rate gives a better optimization time. More precisely, when $x$ denotes the current-best individual, and we choose in the next iteration as mutation rate $p=1/(\LO(x)+1)$, then the expected optimization time decreases to around $(e/4)n^2 \approx 0.68 n^2$. This is almost $21\%$ better than the expected optimization time of the \oea with standard mutation rate $p=1/n$ and about $11.7\%$ better than the mentioned $0.77n^2$ expected running time which the best static mutation rate $p\approx 1.59/n$ achieves.  

Also Randomized Local Search (RLS), the algorithm which flips in each iteration one uniformly selected bit and selects the better of the two offspring as parent individual for the next iteration, can profit from a non-static choice of the \emph{step size}, i.e., the number of bits that it flips in every iteration. It is well known that RLS needs $n^2/2$ iterations, in expectation, to optimize an $n$-dimensional \leadingones instance. In Figure~\ref{fig:SALO} we take a closer look at the optimization process, and plot the expected number of iterations ($y$-axis) needed by RLS to identify, on the $1000$-dimensional \leadingones problem, a solution of fitness value at least $\LO(x)$ ($x$-axis). This is the blue straight line. We also illustrate in the same figure the corresponding expected \emph{fixed-target running times} of the RLS variant which flips in each iteration exactly 2 and 3 pairwise different bits, respectively. These are the yellow and gray curves, respectively. The lower-most, black line illustrates the expected performance of the RLS variant which chooses in each iteration the best of these three parameter values. We observe that this adaptive variant has an expected optimization time that is around $20\%$ smaller than that of standard 1-bit flip RLS. We also see that for $\LO$-values smaller than $n/2$, it is advisable to flip more than one bit per iteration, while 1-bit flips are optimal once a solution of $\LO$-value $\ge n/2$ has been identified. This can be best seen by comparing the slopes of the curves in this plot of fixed-target running times. The ultimate goal of parameter control is the design of mechanisms that detect such transitions and suggest best possible parameter values for the different stages in an automated way. 

\begin{figure}[t]
\begin{center}
\includegraphics[width=0.65\linewidth]{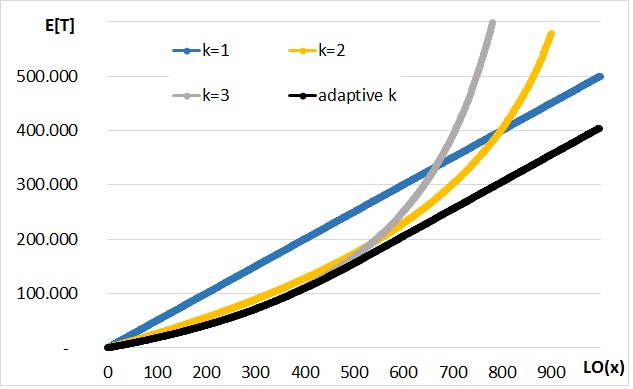}
\end{center}
\caption{Expected fixed-target running times of RLS variants flipping in each iteration exactly 1, 2, 3, or an adaptive number of bits. The adaptive variant, which chooses the best among the three parameter values, has a total expected optimization time that is about 20\% better than RLS, which always flips one bit per iteration}
\label{fig:SALO}
\end{figure}

We note that in the example discussed in this section, ``only'' constant factors could be gained through the dynamic parameter choice, but that in general also asymptotic performance gains can be expected. As example where this has been rigorously proven will be discussed in Section~\ref{sec:SAfitnessTheoretical}.


\section{Classification of Parameter Control Mechanisms}
\label{sec:SAclassification}

A considerable obstacle to overcome when searching for previous works on non-static parameter choices is the lack of a commonly agreed-upon terminology. This has led to a situation in which similar techniques have significantly different names, and, conversely, the same term being used for two fundamentally different concepts. Since 1999 a widely accepted classification scheme for parameter setting has been the taxonomy proposed by Eiben, Hinterding, and Michalewicz in~\cite{EibenHM99}. We present this classification in Section~\ref{sec:SAtaxoalt}, and modify it to cope with the developments in parameter control of the last twenty years in Section~\ref{sec:SAclassi}. 

\subsection{The Classification Scheme of Eiben, Hinterding, and Michalewicz}\label{sec:SAtaxoalt}
Eiben, Hinterding, and Michalewicz~\cite{EibenHM99} distinguish three different types of parameter control, namely deterministic, self-adaptive, and adaptive parameter settings.
\begin{itemize}
	\item A dynamic parameter choice is called \textbf{deterministic} if the choice of the parameter value does not depend on the fitness landscape encountered by the algorithm. Since there is thus no feedback from the optimization process into the parameter choice, the parameter value can only depend on iteration or time counters.
	
	It was noted already in~\cite{EibenHM99} that the term ``deterministic'' is misleading, since a time-dependent parameter choice may still contain randomized elements, that is, the time or iteration counter determines a probability distribution from which the parameter value is sampled. As alternative names for this class of update schemes the terms \emph{scheduled} or \emph{feedback-free parameter control} might be more appropriate.
	\item In \textbf{self-adaptive} parameter choices, the parameters are encoded into the representation of the search points and are thus subject to variation operators. The hope is that the better parameter values yield better offspring and are thus more likely to survive the evolutionary process. By this, implicitly, the choice of the parameters depends on the optimization process and thus, in particular, on the fitness function. 
	\item \textbf{Adaptive} parameter choices are dynamic parameter settings in which there is an explicit dependence of the parameters on the optimization process. This large category includes structurally simple success-based update rules like those resembling the 1/5-th success rule from evolution strategies, but also learning-inspired techniques which choose the parameter values depending on statistics from the optimization process so far. 
\end{itemize}

\subsection{A Revised Classification Scheme}\label{sec:SAclassi}

At the time of writing of~\cite{EibenHM99}, the three different types of parameter control discussed in Section~\ref{sec:SAtaxoalt} were of similar importance. In the last almost twenty years, however, we observe an increasing interest (and massive progress) in the subcategory of adaptive parameter control schemes, which also play a predominant role within the theoretical studies. In particular, the last years made it quite clear that the substantial differences between, say, a simple deterministic fitness-dependent choice of a parameter value and a parameter choice via reinforcement-learning approaches motivate to not have both in the same category. We therefore present in the next subsection an alternative classification scheme, which takes into account this development.

\begin{itemize}
	\item \textbf{State-Dependent Parameter Control.} We classify as \emph{state-dependent parameter control} those mechanisms that depend only on the current state of the search process, e.g., the current population, its fitness values, its diversity, but also a time or iteration counter. Hence this subsumes the previous ``deterministic'' category (containing time-dependent parameter choices) and all other parameter setting mechanisms which determine the current parameter values via a pre-specified function mapping algorithm states to parameter values, possibly in a randomized manner. All these mechanisms require the user to precisely specify how the parameter value depends on the current state and as such need a substantial understanding of the problem to be solved.
%
%
	\item \textbf{Success-Based Parameter Control.} To overcome the usability challenges and the inflexibility of state-dependent parameter control mechanisms, several approaches to set the parameters in a \emph{self-adjusting} manner have been proposed. As one important type of self-adjusting parameter control mechanisms, we classify as \emph{success-based} parameter settings all those mechanisms that change the parameters from one iteration to the next. In other words, the parameter value to be used in the current iteration is determined (possibly in a randomized manner) by the parameter value used in the previous iteration and by an evaluation how successful the previous iteration was. The success measure can be a simple binary information like whether a solution with superior fitness was found, but it could also take into account the quantitative information like the fitness gain or loss in this iteration. Depending on the parameter to be set, also other quantities than the fitness can be taken into account, e.g., the evolution of the diversity of the population.
	
  The most common form of success-based rules are multiplicative updates of parameters, which increase or decrease the parameter value by suitable factors depending on whether the previous iteration was classified as success or not. Success-based rules other than multiplicative updates have been designed as well. For example, in~\cite{DoerrGWY17} the offspring were generated with two different parameter values and the information which parameter value led to the best offspring determined the parameters of the next iteration, cf. Section~\ref{sec:SADoerrGWY17} for a detailed discussion.
	\item \textbf{Learning-Inspired Parameter Control.} As the main second type of self-adjusting parameter control mechanisms, we classify as \emph{learning-inspired parameter control} mechanisms all those schemes which aim at exploiting a longer search history than just one iteration. To allow such learning mechanisms to also adapt quickly to changing environments, older information is taken into account to a lesser extend than more recent ones. This can be achieved by
	only regarding information from (static or sliding) \emph{time windows} or by discounting the importance of older information via weights that decrease (usually exponentially) with the anciency of the data. 
	
	Most learning-inspired parameter control mechanism that have been experimented with in the evolutionary computation context borrow tools from machine learning, where a similar problem known as the \emph{multi-armed bandit problem} is studied, cf. Section~\ref{sec:SAlearning}.
	
	\item \textbf{Endogenous Parameter Control (Self-adaptation).} This category corresponds to the self-adaptive parameter control mechanisms in the taxonomy of~\cite{EibenHM99}. We prefer the name endogenous parameter control as it best emphasizes the structural difference of these mechanisms, which is to encode the parameters in the genome and to let them evolve via the usual variation and selection mechanisms of the evolutionary system. 
	\item \textbf{Hyper-Heuristics.} Hyper-heuristics are algorithms that operate on a set of low-level heuristics, select from it an algorithm, and run it for some time, before re-evaluating which of the low-level heuristics to use next. The main hope is that the hyper-heuristics automate the algorithm selection and configuration process, in a way that allows for maximizing the profit from different algorithmic ideas in the different stages of the optimization process. Similar to the motivation behind endogenous parameter control, the use of a high-level hyper-heuristic is guided by the belief that the high complexity of the parameter control problem calls for efficient heuristic approaches.
\end{itemize}

Figure~\ref{fig:SAclassi} summarizes our classification scheme. Existing theoretical results are summarized in the next sections, which are structured according to this taxonomy.

\begin{figure}[t]
\begin{center}
\includegraphics[width=0.75\linewidth]{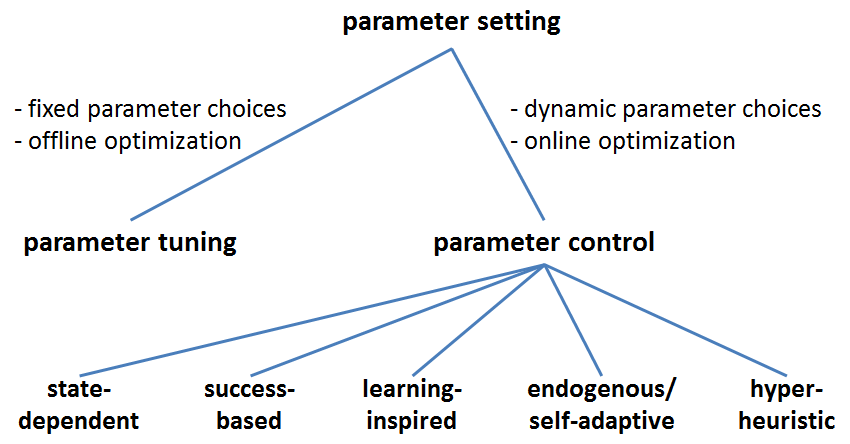}
\end{center}
\caption{Classification of Parameter Control Mechanisms. We call success-based and learning-inspired mechanisms also \emph{self-adjusting}.}
\label{fig:SAclassi}
\end{figure}

We emphasize that our classification is partly driven by the historical development of the field. For example, it would be more logical to not have hyper-heuristics (as long as they essentially optimize parameters) as a separate category, but rather classify them as success-based or learning-inspired parameter control schemes. Since historically the area of hyper-heuristics developed relatively independently (partially due to the fact that there are many hyper-heuristics that cannot be seen as parameter control mechanisms), we prefer to maintain an own category for hyper-heuristics.
 
\section{State-Dependent Parameter Control}\label{sec:SAstate}

We recall from the previous section that state-dependent parameter selection schemes are those mechanisms which choose the parameter values based only on the current state of the algorithm, without making use of the search history. One of the best known examples for state-dependent parameter control is the so-called \emph{cooling schedule} used by Simulated Annealing. The idea of this cooling schedule is to start the heuristic with a rather generous acceptance behavior, and to increase the selective pressure during the optimization process, cf. Section~\ref{sec:SAtime} for a more detailed description. The cooling schedule, as the name suggests, is a time-dependent selection mechanism, which maps the iteration counter to a temperature value that defines the selective pressure. 

As we shall see in this section, time-dependent parameter selection schemes have also been experimented with in the context of evolutionary computation. In addition, other state-dependent parameter settings, like rank- and fitness-based mutation rates and diversity-based parameter choices have been analyzed empirically, but have received considerably less attention in the theory of evolutionary algorithms community. 

\subsection{Time-Dependent Parameter Choices}\label{sec:SAtime}

Simulated Annealing is typically not regarded as an evolutionary algorithm, since it draws inspiration from the physical phenomenon of an annealing process. We nevertheless decided to discuss it in this book chapter, as it is structurally very similar to Randomized Local Search, and certainly falls in the class of iterative randomized black-box optimization heuristics. 

\emph{Simulated Annealing}~\cite{SA83} is a (1+1)-type search heuristic that uses a Boltzmann selection rule to decide whether or not to replace the previous parent individual $x$ by a new solution $y$. More precisely, the algorithm keeps in its memory only one previously evaluated solution $x$, and modifies it by a local variation. In case of pseudo-Boolean maximization this local move is identical to that of RLS, i.e., the offspring $y$ is created from $x$ by flipping exactly one bit, the position of which is chosen uniformly at random. The new solution $y$ always replaces $x$ if it is better, and it replaces $x$ with probability $\exp\big((f(y)-f(x))/T \big)$ otherwise. That is, the better $y$, the larger the probability that it survives the selection procedure. The novelty of Simulated Annealing over its predecessor, the Metropolis algorithm~\cite{Metropolis53}, is an adaptive choice of the ``temperature'' $T$ in the Boltzmann selection rule: while the Metropolis algorithm uses the same $T$ throughout the whole optimization process, the value of $T$ is decreased over time in Simulated Annealing, either with each iteration or, more commonly, after a fixed number $\tau$ of iterations. The adaptive selective pressure results in a more generous acceptance behavior at the beginning of the optimization process (to allow for faster exploration), and a more and more elitist selection towards the end (``exploitation''). Algorithm~\ref{alg:SASA} summarizes this algorithm. For constant $T_t=T$ we obtain from Algorithm~\ref{alg:SASA} the pseudo-code for the Metropolis algorithm. Numerous successful applications and more than 43,000 citations\footnote{This citation number is according to Google Scholar as of April 12, 2018.} of~\cite{SA83} witness that this idea to control the selective pressure during the optimization process can have an impressive impact on the performance. 

\begin{algorithm2e}%
	\textbf{Initialization:} 
	Choose $x \in \{0,1\}^n$ uniformly at random (u.a.r.)\;
 \textbf{Optimization:}
\For{$t=1,2,3,\ldots$}{
Create from $x$ a new solution candidate $y$ by flipping exactly one bit in $x$\;
\eIf{$f(y)\geq f(x)$}
			{$x \assign y$}
			{$x \assign y$ with probability $\exp((f(y)-f(x))/T_t)$}
}
\caption{Simulated Annealing for the maximization of a pseudo-Boolean function $f:\{0,1\}^n \to \R$.}
\label{alg:SASA}
\end{algorithm2e}

A number of theoretical results analyzing the performance of Simulated Annealing exist. Most of these prove convergence to a global optimum for suitably chosen parameter settings, cf. the book chapter~\cite{Henderson2003} for a summary of selected theoretical and empirical results. In addition to results mentioned there, a plethora of running time results exist for combinatorial optimization problems on graphs, including most notably matching~\cite{SasakiH88} and graph bisection problems~\cite{CarsonI01,Impagliazzo01,JerrumS93}. Selected theoretical works that concentrate on the advantages of dynamic parameter choices are summarized below.  

Answering an open problem posed in~\cite{JerrumS96}, Wegener presented in~\cite{Wegener05} a problem class for which Simulated Annealing outperforms its static counterpart, the Metropolis algorithm, regardless of how the temperature value is chosen in the latter. More precisely, Wegener proves that Simulated Annealing with multiplicative temperature decay $T(t)=\alpha T(1)$ ($\alpha<1$ being a constant and the initial value $T(1)$ being ignorant of the instance, but possibly depending on the number of edges $m$ and the maximal weight $w_{\max}$) has a better expected optimization time on some subclasses of the Minimum Spanning Tree (MST) problems than the Metropolis algorithm with any fixed temperature. Previous examples for this phenomenon had been presented in~\cite{Sorkin91} and~\cite{DrosteJW00}, but were of a rather artificial nature. The novelty of~\cite{Wegener05} was thus to prove this statement for a natural combinatorial optimization problem. A particular instance of the MST problem for which Wegener proved the superiority of Simulated Annealing is a graph that has the form of connected triangles. Wegener also showed a provable advantage for $\varepsilon$-separated graphs, in which non-equal weights are apart from each other by a constant factor of at least $1+\varepsilon$, cf.~\cite[Section~5]{Wegener05}.

One of the first works analyzing a classic evolutionary algorithm with a dynamic parameter setting was presented by Droste, Jansen and Wegener in the above-mentioned work~\cite{DrosteJW00}. Besides a \emph{time-dependent selection strategy,} the authors also analyze the \oea with a \emph{time-dependent mutation rate} $p \in \{2^k/n \mid k=0, 1, 2, \dots, \lceil \log_2(n) \rceil -2\}$. In this algorithm, the mutation rate is initialized as $1/n$ and doubled in every iteration until $p$ exceeds $1/2$, in which case it is reset to $1/n$. An example function, \textsc{PathToJump}, is presented for which the \oea with the time-dependent mutation rate needs only $O(n^2 \log n)$ steps, on average, to locate the optimum, while the \oea with static mutation rate $p=1/n$ does not optimize \textsc{PathToJump} in expected polynomial time. The authors also show a converse result in which the dynamic \oea is much slower than the classical static one. It is not difficult to see that the dynamic EA performs worse than the static \oea on most classic benchmark functions like \onemax, \leadingones, etc., cf.~\cite[Section~4]{DrosteJW00}. This work was later extended and simplified by Jansen and Wegener in~\cite{JansenW06}.

In~\cite{JansenW07} a comparison is made between the \oea with static and with time-dependent mutation rates on the one hand, and Simulated Annealing and the Metropolis algorithm on the other hand, but the focus of this work is not on the advantages of adaptive parameter choices, but rather on a comparison of the different selection schemes.


\subsection{Rank-Dependent Parameter Control} 

Motivated by empirical work reported in~\cite{CervantesStephens2009}, Oliveto, Lehre, and Neumann analyzed in~\cite{OlivetoLN09} a $(\mu+1)$ EA with \emph{rank-based mutation rates}. In this algorithm, the individuals of the parent population are ranked according to their fitness values and the mutation rate applied in some iteration $t$ depends on the rank of the (uniformly selected) individual undergoing mutation. The intuition behind this rank-based mutation rates is that individuals at larger ranks (i.e., worse fitness) should be modified more aggressively (suggesting large mutation rates), while the best individuals of the population should be modified with caution, suggesting small mutation rates. 

To be more precise, the algorithm proposed in~\cite{CervantesStephens2009} uses standard bit mutation with mutation rate $p_i$, where for the $i$-th ranked search point the value of $p_i$ is set to $p_{\min}+(p_{\max}-p_{\min})(i-1)/m$ (linear interpolation ensuring a minimal mutation rate of $p_{\min}>0$ and a maximal mutation rate $p_{\max}$). The variant studied in~\cite{OlivetoLN09} uses $p_{\min}=1/n$, $p_{\max}=1$, and $m=\mu$. Theorem~\ref{thm:SAOlivetoLN09} below gives a general upper bound for the rank-based $(\mu+1)$~EA, which is better than the $\Theta(n^n)$ expected running time of the \oea on functions like \textsc{Needle} or \textsc{Trap}. 

\begin{theorem}[Theorems~1 and~2 in~\cite{OlivetoLN09}]
\label{thm:SAOlivetoLN09}
For $\mu\ge 2$ and $\mu=\text{poly}(n)$, the expected optimization time of the $(\mu+1)$ EA with rank-based mutation rates is at most\footnote{This bound is mistakenly stated as $O(2^n)$ in~\cite[Theorem~1]{OlivetoLN09}, but the proof clearly shows the here-stated upper bound.} $7 \cdot 3^n$ for any pseudo-Boolean function $f:\{0,1\}^n \to \R$, and it is $O(\mu n \log n)$ for \onemax.\footnote{We recall that \onemax is the function that assigns to each $x \in \{0,1\}^n$ the number of ones in it; i.e., $\OM(x)=\sum_{i=1}^n{x_i}$. All running time bounds that we state in this chapter for the optimization of \onemax also apply to the optimization of the functions $\OM_z:\{0,1\}^n \to \R, x \mapsto |\{i \in [n] \mid x_i=z_i\}|$, whose fitness landscape is isomorphic to that of $\OM=\OM_{(1,\ldots,1)}$.}  
\end{theorem}

In addition to these results, examples are constructed for which the $(\mu+1)$ EA with rank-based mutation rates performs significantly worse~\cite[Section~V]{OlivetoLN09} and significantly better~\cite[Section~VI]{OlivetoLN09} than the classical $(\mu+1)$ EA with standard bit mutation rate $p=1/n$.

\subsection{Fitness-Dependent Parameter Control}\label{sec:SAfitnessTheoretical}

While rank-based parameter selection had originally been introduced with the hope to find a generally well-functioning control scheme, fitness-based parameter selection schemes are often highly problem-tailored, and cannot be assumed to work particularly well when applied to different objective functions. 
The theoretical results stated below should therefore not be considered as a suggestion for generally applicable parameter control mechanisms, but rather as a point of comparison for more plausible, general-purpose parameter update techniques; i.e., we should use these results only as a lower bound for the performance of a best possible parameter update scheme. This way, the results form a baseline that helps us understand and judge the limits of parameter control. 
	
\subsubsection{Fitness-Dependent Mutation Rates for the (1+1) EA on LeadingOnes} 	
The first work showing a significant advantage of a fitness-dependent choice of the mutation rate has been presented in~\cite{BottcherDN10}, where the following result is shown.\footnote{Prior to~\cite{BottcherDN10}, fitness-dependent mutation rates had also been analyzed in  \emph{immune algorithms}~\cite{Zarges09,Zarges08}, but no advantage of the analyzed parameter choices could be shown.} 

\begin{theorem}[Theorems~3 to~6 in~\cite{BottcherDN10}]
\label{thm:SABottcherDN10}
On \leadingones, the expected number of iterations needed by the \oea with \emph{static} mutation rate $p \in (0,1)$ to identify the optimal solution is $\tfrac{1}{2p^2}[(1-p)^{1-n} - (1-p)]$. This expression is minimized for $p \approx 1.59/n$, which gives an expected optimization time of around $0.77 n^2$.

For the \oea variant that chooses in every iteration the fitness-dependent mutation $p=1/(\LO(x)+1)$, where $x$ denotes the solution that undergoes modification, the expected optimization time decreases to around $(e/4)n^2 \approx 0.68 n^2$. No other fitness-dependent mutation rate can achieve a better expected optimization time.
\end{theorem}
In this result the expected optimization time of the fitness-dependent \oea is almost $21\%$ better than the expected optimization time of the \oea with standard mutation rate $p=1/n$ and about $11.7\%$ better than the $0.77n^2$ expected running time which the best static mutation rate $p\approx 1.59/n$ achieves. 
 
\subsubsection{Fitness-Dependent Mutation Rates for the \texorpdfstring{$(1+\lambda)$~EA}{(1+l)~EA} on OneMax} 	\label{sec:SA1lambdaOM}
	Interestingly, the question how to best control the mutation rate during the optimization process gained relevance with the establishment of \emph{black-box complexity} as a measure for the best possible running time that any randomized search heuristic of a certain type can achieve (cf. \ifthenelse{\equal{\IsChapter}{true}}{Chapter~\ref{chap:BBC} of this book}{\cite{Doerr18BBC}} for a survey of works on this complexity notion). By comparing existing algorithms with the theoretically best possible performance, one can judge how well suited a given approach is. Non-surprisingly, the best-possible algorithms take into account the state of the optimization process, and adjust their parameters accordingly.
	
	In this context, and more precisely, in the context of analyzing lower bounds for the performance of unbiased parallel evolutionary algorithms\ifthenelse{\equal{\IsChapter}{true}}{, cf. Section~\ref{sec:BBCparallel} in Chapter~\ref{chap:BBC} for a more detailed description of the motivation,}{,} Badkobeh, Lehre, and Sudholt analyzed in~\cite{BadkobehLS14} the optimal fitness-dependent mutation rate for the $(1+\lambda)$ EA on \onemax. The main result is summarized by the following theorem. 
	
	\begin{theorem}[Theorems~3 and~4 in~\cite{BadkobehLS14}]\label{thm:SABadkobehLS14}
For $\lambda \le e^{\sqrt{n}}$ the $(1+\lambda)$ EA that uses in each iteration the mutation rate $p(x):=\max\Big\{1/n, \frac{\ln\lambda}{n \ln(en/(n-\OM(x)))} \Big\}$ (where $x$ denotes the parent individual held in the memory at the beginning of the iteration) has an expected optimization time on \onemax equal to $\Theta\Big(n \log n + \frac{\lambda n}{\log \lambda} \Big)$. 

This performance is best possible among all unary unbiased black-box algorithms that create $\lambda$ offspring in parallel. 
	\end{theorem}
	
	The performance of this fitness-dependent $(1+\lambda)$ EA for many values of $\lambda$ is superior to the performance of the $(1+\lambda)$ EA with the static mutation rates regarded so far, which is $\Theta(n \log n + \frac{\lambda n \log\log \lambda}{\log \lambda})$ for mutation rate $p = c/n$, $c$ a constant, see~\cite{DoerrK15,GiessenW17Algorithmica}, and  $\Theta\Big(\sqrt{\lambda} n \log n + \frac{\lambda n}{\log \lambda} \Big)$ for  $p=\ln(\lambda)/(2n)$ and $\lambda \in \omega(1) \cap n^{O(1)}$~\cite[Lemma~1.2]{DoerrGWY17}.
	
	In Section~\ref{sec:SADoerrGWY17} we will see an example for a purely success-based adaptation scheme which achieves the same expected performance as the $(1+\lambda)$ EA with fitness-dependent mutation rate. Most recently, a self-adaptive $(1,\lambda)$ EA has been designed, which also achieves the same bound. This algorithm will be discussed in Section~\ref{sec:SAselfadaptive}.
	 
\subsubsection{Fitness-Dependent Mutation Strengths for RLS on OneMax} 

While the result in Section~\ref{sec:SA1lambdaOM} is of asymptotic order only, one might hope to get more precise results for selected values of $\lambda$. Unfortunately, the precise relationship between function values and optimal mutation rates is not even known in the very special case $\lambda=1$. What is known, however, is the following. 

In~\cite{DoerrDY16} it is shown that the best possible running time on \onemax that any unary unbiased black-box algorithm can achieve is $n \ln(n) - cn \pm o(n)$ for a constant $c$ between $0.2539$ and $0.2665$\ifthenelse{\equal{\IsChapter}{true}}{ (cf. Section~\ref{sec:BBCunbiased} in Chapter~\ref{chap:BBC} for a discussion of unbiased black-box algorithms).}{.} It cannot be better by more than an additive $o(n)$ term than the expected optimization time attained by the RLS variant that chooses in every iteration the mutation strength (i.e., the number of bits to be flipped) in a way that maximizes the expected progress. By the symmetry of the \onemax function, this \emph{drift-maximizing mutation strength} depends only on the fitness of the current-best solution, and not on the structure of this search point. More precisely, when $\ell$ different bits of the search point $x$ are flipped to create $y$, the expected progress $\E[\max\{\OM(y)-\OM(x),0\}]$ equals 
\begin{align}\label{eq:SAdriftOM}
	\sum_{i=\lceil \ell/2 \rceil}^{\ell}
	\frac{\binom{n-\OM(x)}{i} \binom{\OM(x)}{\ell-i}\left(2i-\ell \right)}{\binom{n}{\ell}}.
\end{align}
The drift-maximizing mutation strength $r_{\opt}(x)$ is the value of $\ell$ that maximizes this expression.\footnote{No easy to interpret algebraic relationship between $x$ and $r_{\opt}(x)$ could be established in~\cite{DoerrDY16}, and an approximation of $r_{\opt}(x)$ is therefore used in that work. It is shown, however, that this affects the overall performance by at most $o(n)$ iterations.} 

\begin{theorem}[Theorem~9 in~\cite{DoerrDY16}]\label{thm:SADoerrDY16}
The expected optimization time $\E[T]$ of the drift-maximizing algorithm with fitness-dependent mutation strengths $r_{\opt}(x)$ is $n \ln(n) - cn \pm o(n)$ for a constant $c$ between $0.2539$ and $0.2665$. The unary unbiased black-box complexity is smaller than $\E[T]$ by an additive term of at most $o(n)$.
\end{theorem}

Compared to RLS or the RLS variant using an optimized initialization phase presented and analyzed in~\cite{LaillevaultDD15}, the bound in Theorem~\ref{thm:SADoerrDY16} is smaller by an additive term between $0.138 n \pm o(n)$ and $0.151 n \pm o(n)$. For problem dimensions $\le 10,000$ the advantage of the drift-maximizing algorithm over classic RLS is around $2\%$.

In the language of fixed-budget computation as introduced by Jansen and Zarges in~\cite{JansenZ14} the drift-maximizing algorithm with a budget of at least $0.2675 n$ iterations computes a solution with expected fitness distance to the optimum roughly $13\%$ smaller than the output proposed by RLS~\cite[Section~6]{DoerrDY16}. 

\subsubsection{Fitness-Dependent Offspring Population Sizes in the \texorpdfstring{$(1+(\lambda,\lambda))$ Genetic Algorithm}{(1+(l,l) Genetic Algorithm}} All the results above concern the control of the mutation rate. A fitness-dependent choice of the \emph{offspring population sizes} was considered in~\cite{DoerrDE15} for the \ga on \onemax. Since this algorithm later gave rise to a growing interest in parameter control (note that the conference version~\cite{DoerrDE13} appeared before most of the other results mentioned in this section), we describe this algorithm in more detail. Note in particular that in contrast to the purely mutation-based algorithms mentioned above, the \ga also uses crossover. 

The \ga (Algorithm~\ref{alg:SAga}) works with a parent population of size one. This population $\{x\}$ is initialized with a search point chosen from $\{0,1\}^n$ uniformly at random. The \ga then proceeds in iterations, each consisting of a mutation phase, a crossover phase, and a final elitist selection step determining the new parent population.

In the \emph{mutation phase}, a step size $\ell$ is chosen at random from the binomial distribution $\Bin(n,p)$, where the parameter $p$ is called the mutation rate of the algorithm. Then independently $\lambda$ offspring are created by flipping exactly (i.e., pairwise different) $\ell$ random bits in $x$. In an intermediate selection step, one best mutation offspring $x'$ is selected as mutation winner. In the \emph{crossover phase}, again $\lambda$ offspring are created; this time via a biased uniform crossover between~$x$ and~$x'$, taking each entry from $x'$ with probability $c$ only and taking the entry from $x$ otherwise. Again, an intermediate selection chooses one of the best crossover offspring $y$ as crossover winner. In the final \emph{selection step}, this $y$ replaces $x$ if its \emph{fitness} is at least as large as the fitness of $x$; i.e., if and only if $f(y) \geq f(x)$ holds. 

\begin{algorithm2e}[t]%
	\textbf{Initialization:} 
	Choose $x \in \{0,1\}^n$ uniformly at random (u.a.r.)\;
 \textbf{Optimization:}
\For{$t=1,2,3,\ldots$}{
\underline{\textbf{Mutation phase:}}\\
\Indp
Sample $\ell$ from $\Bin(n,p)$\label{line:L}\;
\lFor{$i=1, \ldots, \lambda$\label{line:mutstart}}{$x^{(i)} \assign \mut_{\ell}(x)$}
Choose $x' \in \{x^{(1)}, \ldots, x^{(\lambda)}\}$ with $f(x')=\max\{f(x^{(1)}), \ldots, f(x^{(\lambda)})\}$ u.a.r.\label{line:mutend}\;
\Indm
\underline{\textbf{Crossover phase:}}\\
\Indp
\lFor{$i=1, \ldots, \lambda$\label{line:costart}}{$y^{(i)} \assign \cross_{c}(x,x')$}
Choose $y \in \{y^{(1)}, \ldots, y^{(\lambda)}\}$ with $f(y)=\max\{f(y^{(1)}), \ldots, f(y^{(\lambda)})\}$ u.a.r.\label{line:coend}\;
\Indm
\underline{\textbf{Selection step:}}
\lIf{$f(y)\geq f(x)$}{$x \assign y$
}
}
\caption{The \ga maximizing a given function $f : \{0,1\}^n \to \R$ with offspring population size~$\lambda$, mutation rate $p$, and crossover bias $c$. The mutation operator $\mut_\ell$ generates an offspring from one parent by flipping exactly $\ell$ random bits (without replacement). The crossover operator $\cross_c$ performs a biased uniform crossover, taking bits independently with probability $c$ from the second argument.}
\label{alg:SAga}
\end{algorithm2e}

The \ga has thus three parameters that need to be set prior to any execution: the offspring population size $\lambda$, the mutation rate $p$, and the crossover bias $c$. Using intuitive considerations, it was suggested in~\cite{DoerrDE15} to use $p=\lambda/n$ and $c=1/\lambda$. With these choices, the 3-dimensional parameter space is reduced to a one-dimensional one, and only $\lambda$ needs to be set. In~\cite{DoerrDE15} it was shown that choosing $\lambda=\Theta(\sqrt{\log n})$ yields an expected running time of $O\left(\max\left\{\frac{n \log(n)}{\lambda}, \lambda n\right\}\right)$ for the \ga on the \onemax problem. This bound was later improved to $F^*=\Theta(n \sqrt{\log(n) \log\log\log(n) / \log\log(n)})$ in~\cite{DoerrD15tight}; this expected running time is attained for a slightly larger value of $\lambda$, namely $\lambda^*=\Theta(\sqrt{\log(n) \log\log(n)/\log\log\log(n)})$. Finally, \cite{Doerr16} showed that the suggested dependencies $p=\lambda/n$ and $c=1/\lambda$ are asymptotically optimal in the sense that any static parameter combination $(p,c,\lambda)$ that gives an expected running time of $O(F^*)$ needs to satisfy $p=\Omega(\lambda^*/n)$, $p = (1/n)\exp(O(\sqrt{{\log(n) \log\log\log(n)} / {\log\log(n)}}\,))$, $c=\Theta(1/(pn))$, and $\lambda=\Theta(\lambda^*)$. No parameter combination can achieve an asymptotically better running time than $\Theta(F^*)$. 

The results mentioned above all concern static parameter values. In terms of dynamic parameters, it was observed already in~\cite{DoerrDE15} that a better expected running time, namely a linear one, can be achieved by the \ga on \onemax if we allow the parameters to depend on the function values. This linear expected performance has later been shown to be asymptotically optimal. 

\begin{theorem}[Theorem~8 in~\cite{DoerrDE15} and Sections~5 and~6.5 in~\cite{DoerrD18ga}]
\label{thm:SAgafitness}
The expected optimization time of the \ga with $p=\lambda/n$, $c=1/\lambda$, and $\lambda=\sqrt{n/(n-f(x))}$ on \onemax is $\Theta(n)$, and this is asymptotically best possible among all dynamic parameter choices. For any static parameter values $(p,c,\lambda)$ the expected running time of the \ga on \onemax is of order at least $n \sqrt{\log(n) \log\log\log(n) / \log\log(n)}$, and thus strictly larger than linear. 
\end{theorem}


In Section~\ref{sec:SAsuccess} we will discuss a success-based parameter control mechanism that identifies and tracks good values for $\lambda$ in an automated way.


\section{Success-Based Parameter Control}\label{sec:SAsuccess}

As success-based parameter control mechanisms we classified all those which change the parameters from one iteration to the next, based on the outcome of the iteration. This includes in particular multiplicative update rules which change parameters by constant factors depending on whether the iteration was considered a success or not. 

\subsection{The 1/5-th Success Rule and Other Multiplicative Success-Based Updates}\label{sec:SARechenberg}

Already the very early works on evolution strategies used a simple, yet powerful technique to adapt the parameters online. The so-called \emph{1/5-th success rule}, which was independently discovered in~\cite{Rechenberg, Devroye72, SchumerS68}, suggests to set the step size of an evolution strategy in such a manner that $1/5$-th of the iterations lead to a fitness improvement. The idea behind this is that when the success rate is higher, then most likely the step size is too small and time is wasted on minor improvements; however, when the success rate is smaller, then time is wasted by waiting too long for an improvement. The value $1/5$ was derived from some theoretical considerations for the performance of the (1+1) evolution strategy on the \emph{sphere} problem $f:\R^n \to \R, x \mapsto \sum_{i=1}^n{x_i^2}$. Rechenberg showed that a success rate of about $20\%$ yields optimal expected gain for this problem (and also on another problem with a so-called inclined ridge, cf.~\cite{Rechenberg} for details).  
%

The first implementations of this 1/5-th success rule were not success-based in our language, but rather observed the success rate over several iterations and then adjusted the step size if a discrepancy from the target success rate of $1/5$ was detected. In~\cite{KernMHBOK04}, a simpler success-based implementation was proposed. Here, the step size is multiplied by some number $F > 1$ in case of success and divided by $F^{1/4}$ in case of no success. The hyper-parameter $F$ is called the \emph{update strength} of the adaptation rule. 

%
%

We next present two examples for success-based parameter control suggested in the literature. 

\textbf{Example 1: the 1/5-th success rule applied to the \ga.} It may be surprising that a simple multiplicative success-based rule can work. We therefore present an illustrated example, the self-adjusting \ga, which has originally been proposed in~\cite{DoerrDE15} and later been formally analyzed on the \onemax problem in~\cite{DoerrD15self}. We will describe this algorithm in more detail in Section~\ref{sec:SAgaadaptive}, but note here only that by using the recommended dependencies $p=\lambda/n$ and $c=1/\lambda$ the self-adjusting \ga requires to set the offspring population size $\lambda$ as only parameter. The value of $\lambda$ is adapted based on the success of a full iteration, using the above-sketched implementation of the 1/5-th success rule suggested in~\cite{KernMHBOK04}. Figure~\ref{fig:SAgalambda} shows how well the optimal fitness-dependent value of the offspring population size $\lambda$ suggested by Theorem~\ref{thm:SAgafitness} (smooth black curve) is approximated by this multiplicative success-based update rule (staggered red curve). The uppermost (blue) curve shows the evolution of the current-best fitness value, from which the optimal fitness-dependent mutation rate is computed. Note that in this figure we show the optimal mutation rates \emph{per iteration}, each of which costs $2 \lambda$ function evaluations. The update strength $F$ in this illustration is set to $1.5$.

\begin{figure}[!h]
\begin{center}
\includegraphics[width=0.75\linewidth]{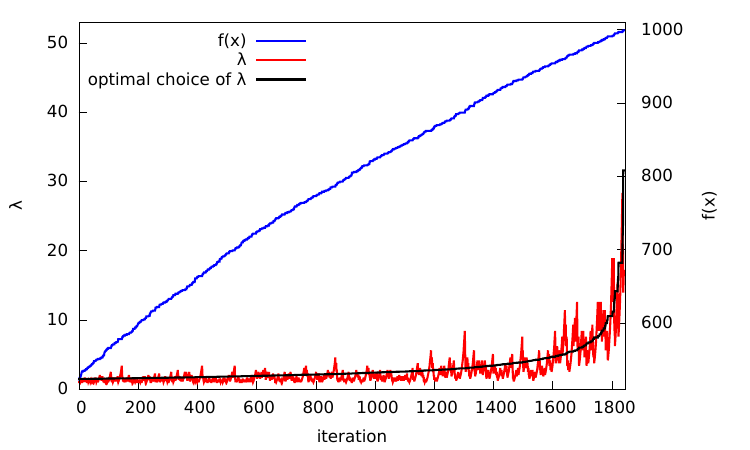}
\end{center}
\caption{Application of the $1/5$-th success rule to the offspring population size in the \ga on \onemax}
\label{fig:SAgalambda}
\end{figure}
%

\textbf{Example 2: The \ocl with success-based offspring population size $\lambda$.} A different success-based parameter control has been suggested in~\cite{HansenGO95} for the control of the offspring population size $\lambda$ in a non-elitist $(1,\lambda)$~evolution strategy (ES). Motivated by a theoretical result that proves that in the $(1,\lambda)$~ES the so-called local serial progress is maximized when the expected progress of the second best offspring created in one iteration is zero (this result applies to any function $f:\R^n \to \R$), the following multiplicative success-based update rule for the offspring population size $\lambda$ has been suggested. Denoting by $x^{(t)}$ the parent individual of the $t$-th iteration, by $\lambda(t)$ the selected offspring population size, and by $x^{(t),1}, \ldots, x^{(t),\lambda(t)}$ the offspring created in the $t$-th iteration, sorted by decreasing function values, then the offspring population size for the next iteration is set to 
\begin{equation}\label{eq:SAhansen}
\lambda^{(t+1)}:= \max\left\{2, \lambda^{(t)} \exp\left(\frac{-\beta (f(x^{(t),2})-f(x^{(t)}))}{\sqrt{\sum_{i=1}^{\lambda(t)}{(f(x^{(t,)i})-f(x^{(t)}))^2/(\lambda-1)}}}\right)\right\},
\end{equation} 
where $\beta \in (0,1)$ is a hyper-parameter that controls the speed of the adaptation. While this update mechanism, to the best of our knowledge, has not been formally analyzed, it is shown in~\cite{HansenGO95} to give good performance on the hyper-plane and the hyper-sphere problem.

\subsection{Theoretical Results for Success-Based Parameter Control}\label{sec:SAsuccesstheo}

In this section we describe the theoretical results known for success-based based parameter control mechanisms. We note that some works on hyper-heuristics resemble closely a success-based parameter control. The reader can find these in Section~\ref{sec:SAHHadv}.

\subsubsection{The Self-Adjusting \texorpdfstring{$(1+(\lambda,\lambda))$~GA}{(1+(l,l)) GA} on OneMax and on MaxSAT}
\label{sec:SAgaadaptive} 

We have seen in Theorem~\ref{thm:SAgafitness} that the \opllga with mutation rate $p=\lambda/n$, crossover bias $c=1/\lambda$, and fitness-dependent population size $\lambda=\sqrt{n/(n-f(x))}$ takes an expected number of $\Theta(n)$ function evaluations to optimize a \onemax instance of problem dimension $n$. This is the asymptotically best running time among all static and dynamic parameter choices. A substantial drawback of this result is the rather complex dependence of $\lambda$ on the current-best function value $f(x)$. The question whether this relationship can be detected by a parameter control mechanism in an automated way suggests itself. In fact, already in~\cite{DoerrDE15} a success-based choice of $\lambda$ was suggested, and shown to achieve a very similar empirical performance as the fitness-dependent choice, across all tested problem dimensions $n \le 5,000$. In~\cite{DoerrD18ga} the efficiency of this success-based variant of the \opllga, which we will describe in more detail below, could be formally proven. 
\begin{theorem}[Theorem~9 in~\cite{DoerrD18ga}]
\label{thm:SAga15}
The expected optimization time of the self-adjusting \ga (Algorithm~\ref{alg:SAgaself}) with mutation rate $p=\lambda/n$, crossover bias $c=1/\lambda$, and sufficiently small update strength $F>1$ on \onemax is $\Theta(n)$.
\end{theorem}

\begin{algorithm2e}[t]%
\textbf{Initialization:} 
		Sample $x \in \{0,1\}^n$ uniformly at random (u.a.r.)\;
		Initialize $\lambda \assign 1$\;
\textbf{Optimization:}
\For{$t=1,2,3,\ldots$}{
\underline{\textbf{Mutation phase:}}\\
\Indp
	Sample $\ell$ from $\Bin(n,p)$\;
	\lFor{$i=1, \ldots, \lambda$}{$x^{(i)} \assign \mut_{\ell}(x)$}
	Choose $x' \in \{x^{(1)}, \ldots, x^{(\lambda)}\}$ with $f(x')=\max\{f(x^{(1)}), \ldots, f(x^{(\lambda)})\}$ u.a.r.\;
\Indm
\underline{\textbf{Crossover phase:}}\\
\Indp
\lFor{$i=1, \ldots, \lambda$}{$y^{(i)} \assign \cross_{c}(x,x')$}
Choose $y \in \{y^{(1)}, \ldots, y^{(\lambda)}\}$ with $f(y)=\max\{f(y^{(1)}), \ldots, f(y^{(\lambda)})\}$ u.a.r.\;
\Indm
\underline{\textbf{Selection and update step:}}\\
\Indp
\lIf{$f(y)>f(x)$}{
$x \assign y$; $\lambda \assign \max\{\lambda/F,1\}$}
\lIf{$f(y)=f(x)$}{
$x \assign y$; $\lambda \assign \min\{\lambda F^{1/4},n\}$}
\lIf{$f(y)<f(x)$}{$\lambda \assign \min\{\lambda F^{1/4},n\}$}
\Indm
}
\caption{The self-adjusting \ga with mutation probability $p$, crossover bias~$c$, and update strength $F$.}
\label{alg:SAgaself}
\end{algorithm2e}

The success-based choice of the parameter $\lambda$ uses the above-mentioned implementation of the 1/5-th success rule considered in~\cite{KernMHBOK04}. That is, after an iteration that led to an increase of the best observed function value (``success''), $\lambda$ is reduced by a constant factor $F > 1$. If an iteration was not successful, $\lambda$~is increased by the multiplicative factor $F^{1/4}$. Consequently, after a series of iterations with an average success rate of $1/5$, this mechanism ends up with the initial value of $\lambda$. 

Since $p=\lambda/n$, the value of $\lambda$ is capped at $n$. Likewise, it is capped from below at $1$. The value of $\lambda$ is allowed to be non-integral. Where an integer is required (i.e., in lines~6,7,9, and~10 of Algorithm~\ref{alg:SAgaself}), $\lambda$~is rounded to its closest integer. That is, in these four lines, instead of $\lambda$ we regard $\lfloor \lambda \rfloor= \lambda - \{\lambda\}$ if the fractional part $\{\lambda\}$ of $\lambda$ is less than $1/2$ and we regard $\lceil \lambda \rceil=\lfloor \lambda\rfloor +1$ otherwise. 

In the experiments conducted in~\cite{DoerrDE13}, see in particular Figure~4 there, all update strengths $F \in [1,2]$ worked well. While this indicates some robustness of the result in Theorem~\ref{thm:SAga15} with respect to the $F$-value, it has been argued in~\cite[Section~6.4]{DoerrD18ga} that update strengths $F$ greater than $2.25$ may lead to an exponential expected optimization time on \onemax.  A commonly used value for $F$, also used in Auger's implementation~\cite{Auger09}, is $F = 1.5$. This is also the value with which Figure~\ref{fig:SAgalambda} has been created. 

One may further wonder how important is the relationship of the two multiplicative updates, that is, the exponent $1/4$. It is argued in~\cite[Section~6.4]{DoerrD18ga} that a similar result as in Theorem~\ref{thm:SAga15} is likely to hold for a range of other exponents as long as the exponent is not too large. Hence in discrete optimization, there is no particular reason for a $1/5$-th rule (that is, the exponent $1/4$). 
This has also been observed in a recent work on image composition, where a success-based $1/k$-th success rule was used to adjust the length of a random walk that is part of the mutation operator~\cite{NeumannSCN17image}. In a set of initial experiments $k=9$ seemed to be a suitable value, and is used for the empirical evaluations.    

Being the first algorithm which can provably reduce the expected optimization time by applying a success-based parameter control mechanism, the self-adjusting \ga has been analyzed also on other functions, by empirical and theoretical means. Already in~\cite[Section~4]{DoerrDE15} a promising empirical performance for linear functions $f:\{0,1\}^n \to \R, x \mapsto \sum_{i=1}^n{w_i x_i}$ with random weights $w_i\in [1,2]$ and for so-called royal road functions with block size $5$ was reported. 
In~\cite{GoldmanP14} the self-adjusting \ga is tested on a number of combinatorial problems. In particular for the maximum satisfiability problem, the self-adjusting \ga shows a very good performance, beaten only by the parameterless population pyramid proposed in the same work. Inspired by this result, a mathematical running time analysis of the \ga on random satisfiability instances was conducted in~\cite{BuzdalovD17}. It confirms that the \ga has a better performance than solely mutation-based algorithms, see, e.g.,~\cite{DoerrNS17}. The work however also shows that weaker fitness-distance correlation of the satisfiability instances can lead to the effect that when offspring are created with a high mutation rate, then the algorithm has problems  determining the structurally better ones. This difficulty can be overcome by imposing an upper limit on the population size $\lambda$, which determines the mutation rate $p = \lambda /n $.

\subsubsection{The \texorpdfstring{\opl}{(1+l) EA} with Success-Based Offspring Population Size \texorpdfstring{$\lambda$}{l}} 
\label{sec:SALassigS11}

For the \opl, the following success-based adaptation of the offspring population size has been suggested in~\cite[Section~5]{JansenJW05}. The offspring population size $\lambda$ is initialized as one. After each iteration, the number $s$ of offspring having a function value that is at least as large as that of the parent fitness is determined. When $s=0$ (i.e., if the iteration has been unsuccessful), the offspring population size $\lambda$ is doubled, otherwise it is replaced by $\lfloor \lambda/s \rfloor$. The intuition for this adaptive choice of the offspring population is to have the value of $\lambda$ inversely proportional to the probability of creating an offspring that replaces its parent. This algorithm, which we call the $(1+\{2\lambda,\lambda/s\})$~EA, had not been analyzed by mathematical means in~\cite{JansenJW05}, but showed encouraging empirical performance on \onemax, \leadingones, and a benchmark function called \textsc{SufSamp}.

The idea of a success-based offspring population size was taken up in~\cite{LassigS11}, where a theoretical analysis of two similar success-based update schemes was performed. The first update scheme, the $(1+\{2\lambda,1\})$~EA, doubles $\lambda$ in case no strictly better search point could be identified and sets $\lambda$ to one otherwise. The second \opl variant, the $(1+\{2\lambda,\lambda/2\})$~EA, also doubles $\lambda$ if no solution of quality better than the parent is found, and reduces $\lambda$ to $\max\{1,\lfloor \lambda/2 \rfloor\}$ otherwise. While these schemes do not result in an improved overall running time in terms of function evaluations, they are both able to achieve a significant reduction of the \emph{parallel} optimization time on selected benchmark problems. That is, the average number of \emph{generations} needed before an optimal solution is evaluated for the first time is smaller than that of classical sequential EAs, which do not perform any evaluations in parallel. The precise results are as follows.
\begin{table}
\begin{center}
\begin{tabular}{c|c|c|l}
Function & Algorithm & $\E[T^{\text{seq}}]$ & $\E[T^{\text{par}}]$ \\
\hline
\onemax & $(1+\{2\lambda,1\})$~EA & $\Theta(n \log n)$     &  $O(n)$ [*]       \\
& $(1+\{2\lambda,\lambda/2\})$~EA &   $\Theta(n \log n)$           &  $O(n)$ \\
\hline
\leadingones & $(1+\{2\lambda,1\})$~EA &  $\Theta(n^2)$    &      $\Theta(n \log n)$   \\
& $(1+\{2\lambda,\lambda/2\})$~EA &    $\Theta(n^2)$          &  $O(n)$ \\
\hline
unimodal with $d$ different  & $(1+\{2\lambda,1\})$~EA &  $O(dn)$    &  $O(d \log n)$       \\
function values & $(1+\{2\lambda,\lambda/2\})$~EA &     $O(dn)$         &  $O(d+ \log n)$ \\
\hline
$\textsc{jump}_k$, $k\ge 2$ & $(1+\{2\lambda,1\})$~EA & $O(n^k)$     & $O(n+k \log n)$ [*]  \\
& $(1+\{2\lambda,\lambda/2\})$~EA &     $O(n^k)$         &  $O(n+k \log n)$ \\
\hline
\end{tabular}
\caption{Expected sequential and parallel running times of the $(1+\{2\lambda,\lambda/2\})$~EA and the $(1+\{2\lambda,1\})$~EA on selected benchmark problems~\cite{LassigS11}. For the two bounds marked [*], we slightly improve the original bound of $O(n \log n)$ via an elementary argument, cf. proof below Theorem~\ref{thm:SALassigS11}}
\label{tab:SALassigS11}
\end{center}
\end{table}

\begin{theorem}[Theorem~7 in~\cite{LassigS11} and proof below for the results marked [*] in Table~\ref{tab:SALassigS11}]
\label{thm:SALassigS11}
The sequential and parallel expected running time of the $(1+\{2\lambda,\lambda/2\})$~EA and the $(1+\{2\lambda,1\})$~EA satisfy the bounds given in Table~\ref{tab:SALassigS11}.
\end{theorem}

\begin{proof}
Using the classic fitness level method, the expected parallel running time of the $(1+\{2\lambda,1\})$~EA on \onemax is bounded from above by $2 \sum_{i=1}^{n-1}{\log(\tfrac{2en}{n-i})}$ in~\cite{LassigS11}. This expression is further bounded by $2n \log(2en)=O(n \log n)$. However, a closer look reveals that with Stirling's formula, we easily obtain
\begin{align*}
2 \sum_{i=1}^{n-1}{\log\Big(\frac{2en}{n-i}\Big)} 
\le 
2 \log\Big(\frac{(2en)^n}{n!}\Big)
\le 
2 \log\Big(\frac{(2en)^n}{(n/e)^n}\Big)
= 
2 \log((2e^2)^n)
= 
O(n).
\end{align*}
This improved bound immediately carries over to the bound for $\textsc{jump}_k$, $k\ge 2$, where the expected parallel running time of the $(1+\{2\lambda,1\})$~EA is bounded by the expected parallel running time on \onemax plus the time needed to ``jump'' from the local optimum to the global one, which is of order at most $k \log n$. 
\end{proof}

\subsubsection{The 2-Rate \texorpdfstring{\opl}{(1+l)~EA} with Success-Based Mutation Rates}\label{sec:SADoerrGWY17}

In the previous examples, we have studied different ways to control the offspring population size. We now turn our attention to a success-based adaptation of the mutation rates in a \opl with fixed offspring population size $\lambda$, which has been introduced and analyzed in~\cite{DoerrGWY17}. The \opladap stores a parameter $r$ that controls the mutation rate. This parameter is adjusted after each iteration by the following mechanism. In each iteration, the \opladap creates $\lambda/2$ offspring by standard bit mutation with mutation rate $r/(2n)$, and it creates $\lambda/2$ offspring with mutation rate $2r/n$. At the end of the iteration a random coin is flipped. With probability $1/2$ the value of $r$ is replaced randomly by either $r/2$ or $2r$ and with the remaining $1/2$ probability it is set to the value that the winning individual of the last iteration has been created with. Finally, the value $r$ is capped at $2$ if it smaller, and at $n/4$, if it exceeds this value. Algorithm~\ref{alg:SAopladaptive} summarizes this 2-rate \opl variant.

 \begin{algorithm2e}%
	\textbf{Initialization:} 
	Sample $x \in \{0,1\}^{n}$ uniformly at random\;
	Initialize $r \assign r^{\text{init}}$\;
  \textbf{Optimization:}
	\For{$t=1,2,3,\ldots$}{
		\For{$i=1,\ldots,\lambda/2$}{ 
			Create $y^{(i)}$ by flipping each bit in $x$ independently with probability $r/(2n)$\;
		 }
		\For{$i=\lambda/2+1,\ldots,\lambda$}{
		Create $y^{(i)}$ by flipping each bit in $x$ independently with probability $2r/n$\;
		 }
		$x^* \assign \arg\max\{f(y^{(1)}), \ldots, f(y^{(\lambda)})\}$ (ties broken u.a.r.)\;
		\lIf{$f(x^*)\ge f(x)$}{$x \assign x^*$}
		With prob. $1/4$ replace $r$ by $\max\{r/2,2\}$, with prob. $1/4$ by $\min\{2r,n/4\}$, and the remaining prob. replace $r$ by the probability that $x^{*}$ has been created with (capped again at $2$ and $n/4$, respectively)\; 
	}
\caption{The 2-rate \opladap with adaptive mutation probabilities and static population size for the maximization of a pseudo-Boolean function $f:\{0,1\}^n \rightarrow \R$}
\label{alg:SAopladaptive}
\end{algorithm2e}  

\begin{theorem}[Theorem~1.1 
in~\cite{DoerrGWY17}]\label{thm:SADoerrGWY17}
Let $\lambda=\omega(1)$ and $\lambda=n^{O(1)}$. The expected optimization time of the \opladap on \onemax is $\Theta(n \log n + n\lambda/\log \lambda)$. 
\end{theorem}
By the result presented in Theorem~\ref{thm:SABadkobehLS14} above, the $\Theta(n \log n + n\lambda/\log \lambda)$ expected running time achieved by the \opladap is best possible among all $\lambda$-parallel black-box algorithms.

\subsubsection{Success-Based Mutation Strengths for the Multi-Variate OneMax Problem} 

In~\cite{DoerrDK16PPSN} a success-based choice of the mutation strength has been proven to be very efficient for a multi-variate generalization of the \onemax problem. Concretely, the authors study three different classes of generalized \onemax functions. Denoting the size of the alphabet by $r$, the first class contains, for all $z \in [0..r-1]^n$, the functions 
$\OM^{(1)}_z: [0..r-1]^n \to [0..n]; x \mapsto |\{i \in [1..n] \mid x_i=z_i\}|$, the second all functions
$\OM^{(2)}_z: [0..r-1]^n \to [0..n(r-1)]; x \mapsto \sum_{i=1}^n |x_i-z_i|$, while the third class subsumes all $r^n$ functions
$\OM^{(3)}_z: [0..r-1]^n \to [0..n(r-1)]; x \mapsto \min\{x_i - (z_i-r), |x_i-z_i|, (z_i+r)-x_i\}$.
Unlike all other settings regarded in this chapter, \cite{DoerrDK18} studies the \emph{minimization} of these \onemax generalizations. In our description below, we stick to this optimization target, to ease the comparison with the original publication.

The self-adjusting algorithm studied in~\cite{DoerrDK16PPSN} is a RLS-variant, which flips one coordinate in every iteration. For each coordinate $i$, a \emph{velocity} $v_i$ is stored, which denotes the mutation strength at this coordinate. When in iteration $t$ coordinate $i$ is chosen for modification, the entry $x_i$ of the current-best solution $x$ is replaced by $x_i - \lfloor v_i \rfloor$ with probability $1/2$ and by $x_i + \lfloor v_i \rfloor$ otherwise. The entries in positions $j \neq i$ are not subject to mutation. The resulting string $y$ replaces $x$ if its fitness is at least as good as the one of $x$, i.e., if $f(y) \leq f(x)$ holds (we recall that we aim at minimizing $f$). If the offspring $y$ is strictly better than its parent $x$, i.e., if $f(y)<f(x)$ holds, the velocity $v_i$ in the $i$-th component is increased by multiplying it with a fixed constant $A>1$ and $v_i$ is decreased to $bv_i$ otherwise, where $b<1$ is again some fixed constant. If the value of $v_i$ drops below 1 or exceeds $\lfloor r/4 \rfloor$, it is capped at these values.

\begin{theorem}[Theorem~17 in~\cite{DoerrDK18}]
\label{thm:SADoerrDK17algo}
For constants $A,b$ satisfying $1<A \leq 2$, $1/2 < b \leq 0.9$, $2Ab-b-A>0$, $A+b>2$, and $A^2b>1$ the expected running time of the $\RLS_{A,b}$ on any of the generalized $r$-valued \OneMax function $\OM^{(i)}_z$, $i \in \{1,2,3\}$ and $z \in [0..r-1]^n$, is $\Theta(n(\log n + \log r))$. This is asymptotically best possible among all comparison-based variants of RLS and the \oea.
\end{theorem}
In this theorem, the update strengths can be chosen, for example, as $A\in [1.6,2]$ and $b=(1/A)^{1/4}$, imitating the above-mentioned interpretation of the 1/5-th success rule proposed in~\cite{KernMHBOK04}.
 
Using a result proven in~\cite{Dit-Row-Weg-Woe:j:10}, it is argued in~\cite[Section~6.1]{DoerrDK18} that the $\Theta(n(\log n + \log r))$ expected running time of the self-adaptive RLS variant is better by a multiplicative factor of at least $\log r$ than any RLS or \oea variant using \emph{static} step sizes. The optimality of the bound follows from the simple information-theoretic $\Omega(n \log r)$ lower bound which applies to all comparison-based algorithms, while the $\Omega(n \log n)$ lower bound applies to any unary unbiased black-box algorithm\ifthenelse{\equal{\IsChapter}{true}}{ (cf. Section~\ref{sec:BBCunbiased} of Chapter~\ref{chap:BBC} for a discussion of unbiased black-box algorithms).}{.}

\subsubsection{Success-Based Migration Intervals for Parallel EAs in the Island Model}
\label{sec:SAMambriniS15}

A multiplicative success-based adaptation scheme has also been used to adjust the migration interval in a parallel (1+1) EA with a fixed number of $\lambda$ islands. Mambrini and Sudholt~\cite{MambriniS15} apply the two schemes described in Section~\ref{sec:SALassigS11} for the control of the offspring population size of the \opl now to control the migration interval. In their parallel EA, every island has its own migration interval at the end of which it broadcasts its current-best solution to all of its neighbors. In the $(2\tau_i,1)$ variant of the parallel EA (Algorithm~2 in~\cite{MambriniS15}), improved solutions are always broadcast instantly, to all neighboring islands, and the migration interval $\tau_i$ of the corresponding island is set to one. It is set to one also if during the migration interval at least one superior solution has migrated to the island. The migration interval is doubled otherwise, i.e., if at the end of the migration period no strictly better solution has been identified or migrated from a different island. 

In the $(2\tau_i,\tau_i/2)$ scheme (Algorithm~3 in~\cite{MambriniS15}), the broadcast happens only at the end of the migration interval, which is again doubled in case no improved solution could be identified nor migrated from another island, and halved otherwise. 

The $(2\tau_i,\tau_i/2)$ scheme is analyzed for the complete graph topology, for which all migration intervals $\tau_i$ are identical. For the $(2\tau_i,1)$ variant \cite{MambriniS15} proves results for general graph topologies with $\lambda$ islands as well as for a few selected topologies like the unidirectional ring, the grid, a torus, etc. The results comprise upper bounds on the expected communication effort needed to optimize general black-box optimization benchmarks, cf. Sections~4 and~5 in~\cite{MambriniS15}. These bounds are then applied to the same benchmark functions as those regarded in~Theorem~\ref{thm:SALassigS11}. In some cases, including the complete graph topology, the adaptive migration intervals are shown to outperform any static choice in terms of expected communication effort, without (significantly) increasing the expected parallel running time. Table~1 in~\cite{MambriniS15} summarizes the results for the selected benchmark problems. The bounds proven in~\cite{MambriniS15} are upper bounds, and the question of complementing these with meaningful lower bounds seems to remain an open problem. 


\section{Learning-Inspired Parameter Control}

In contrast to the success-based control mechanisms discussed in the previous section, we call \emph{learning-inspired} all those self-adjusting parameter control mechanisms which are based on information obtained over more than one iteration. 

\subsection{Adaptive Operator Selection}

An important class of parameter control schemes takes inspiration from the machine learning literature, and in particular from the \emph{multi-armed bandit problem}. These \emph{adaptive operator selection} techniques\footnote{The term ``operator'' is used because the adaptive operator selection mechanisms have originally not only been designed to choose between different parameter values but also between different actions, such as different variation operators.} maintain a portfolio of $k$ possible parameter values. At each step they decide which of the possible parameter values to use next. To this end, they assign to each possible parameter value a \emph{confidence value}. This confidence value is supposed to be an indicator for how suitable the corresponding value is at the given stage of the optimization process. The confidence can be, for example, an estimator for the likelihood or the magnitude of progress we would obtain from running the algorithms with this value. These confidence values determine or modify the \emph{probabilities} of choosing the corresponding parameter value. We present below three ways to implement this adaptive operator selection principle. 

What distinguishes the parameter control setting from the classically regarded setting in machine learning is the fact that the ``rewards'', i.e., the gain that we can obtain with a given value, can drastically change over time, compared to the static (but random) reward typically investigated in the machine learning literature. The non-static reward distributions change the complexity of the algorithms and the theoretical analysis considerably. As far as we know, the only theoretical work rigorously proving an advantage of learning-based parameter control is~\cite{DoerrDY16PPSN}, which we shall discuss in more detail in Section~\ref{sec:SAlearning}. Despite the promising empirical performance of adaptive operator selection techniques, none of the techniques mentioned below could establish itself as a standard routine. Potential reasons for this situation include the complexity of these techniques, the difficulty of finding good hyper-parameters that govern the update rules, and a lack of theoretical support. 

\begin{itemize}
	\item \textbf{Probability Matching.} This technique aims at assigning the probabilities proportionally to the confidence values, while maintaining for each parameter value a minimal probability $p_{\min}$ for being sampled. Concretely, in round $t$ we choose the $i$-th parameter value with probability 
$$p^i_{t}:= p_{\min} + (1-k p_{\min}) \frac{c_{t}^i}{\sum_{j=1}^k{c_{t}^j}},$$
where $k$ is the total number of different parameter values from which we can choose (the size of the \emph{portfolio}) and $c_t^j$ the confidence in parameter value $j$ at time $t$. 

After executing one iteration with the $i$-th parameter value, its confidence value is updated to 
$$c_{t+1}^i := (1-\alpha) c_t^i + \alpha r^t,$$ 
where $r^t$ denotes the (normalized) reward obtained in the $t$-th round and $0<\alpha<1$ is the \emph{hyper-parameter} that determines the speed of the adaptation. The confidence value of parameter values that have not been selected in the $t$-th round are not updated. 

\item \textbf{Adaptive Pursuit.} When larger portfolios used, the previous mechanism choosing the operator with probability roughly proportional to the confidence value might not give enough preference for the truly best choice. To this aim, a more ``aggressive'' update rule has been suggested: \emph{adaptive pursuit.} This selection scheme uses the same confidence values as Probability Matching, but applies a much more progressive update rule for the probabilities.  In Adaptive Pursuit the probabilities of selection are obtained from the probabilities of the previous iteration according to a ``the winner takes it all'' policy. Concretely, the ``best'' arm, i.e., the parameter value with the highest confidence value is assigned a probability of $(1-\beta) p^{i^*}_{t}+\beta p_{\max}$, while for all other parameters the probability of being sampled is reduced to $p_{t+1}^i:=(1-\beta) p^{i}_{t}+\beta p_{\min}$. Empirical comparisons of Probability Matching and Adaptive Pursuit are presented in~\cite{Thierens05}. In general, it seems that Adaptive Pursuit is more suitable for situations in which the quality differences between the potential parameter values are small, but persistent. 

\item \textbf{Upper Confidence Bound.} The \emph{upper confidence bound} (UCB)-algorithm, originally proposed in~\cite{Auer2002}, plays an important role in machine learning, as it is one of the few strategies that can be proven to behave optimally in a classical operator selection problem. More precisely, the UCB algorithm can be proven to achieve minimal \emph{cumulative regret} in the \emph{multi-armed bandit problem} in which the reward of each ``arm'' follows a static probability distribution. Interpreting the different ``arms'' as the different parameter values that we want the algorithm to choose from, the UCB algorithm chooses in every step the parameter value $i$ that maximizes the expression
$$ \text{ER}(i) + \sqrt{c \log\Big( \frac{2\sum_{j=1}^k{n_{j,t}}}{n_{i,t}}\Big)},$$
where $\text{ER}(i)$ is an estimate for the expected reward of the $i$-th parameter value, $n_{j,t}$ is the number of times the $j$-th parameter value has been chosen in the first $t$ iterations, and $c$ is a hyper-parameter that determines the bias between exploiting parameter values with high expected reward and exploring parameter values that have not yet been tested very often. While being provably optimal in static settings, the UCB algorithm is rather sedate, and thus not very well suited for environments that gradually change over time---the typical situation encountered in the optimization of rather smooth optimization problems. In the parameter control context, it therefore makes sense to replace $n_{j,t}$ by an index that counts the number of occurrences in a given time interval only, instead of considering the whole history (\emph{sliding window}, cf.~\cite{FialhoCSS10} for a detailed discussion and experimental results on two discrete benchmark problems). In contrast, when the environments change abruptly, a combination of the UCB algorithm with a statistical test that detects significant changes in the fitness landscape has been shown to perform very well~\cite{DaCostaGECCO08,FialhoCSS09}.
\end{itemize}

\subsection{Theoretical Results for Learning-Inspired Parameter Control}\label{sec:SAlearning}

The first, and so far only, theoretical work that rigorously analyzes a learning-inspired parameter selection scheme is~\cite{DoerrDY16PPSN}. The algorithm proposed there is a generalized version of \emph{randomized local search (RLS)}, which selects in every step the number of bits to be flipped according to the following rule. With probability $\varepsilon>0$ a random one of the $k$ possible mutation strengths $1,\ldots,k$ is chosen, and with the remaining probability the algorithm greedily selects the parameter value for which the expected progress (coined \emph{velocity} in~\cite{DoerrDY16PPSN}) is maximized. The expected progress is estimated by a time-discounted average of the progresses observed in the learning iterations. More precisely, the velocity of mutation strength $r$ at time $t$ is defined via 
\begin{equation}
v_t[r] := \frac{\sum_{s=1}^t \mathbf{1}_{r_s = r} (1-\delta)^{t-s} (f(x_s) - f(x_{s-1}))}{\sum_{s=1}^t \mathbf{1}_{r_s = r} (1-\delta)^{t-s}},
\label{eq:velo}
\end{equation}
where $r_s$ is the parameter value used in the $s$-th iteration, and the hyper-parameter $\delta$ determines the speed of the adaptation process. \cite{DoerrDY16PPSN} refer to $\delta$ as the \emph{forgetting rate}, inspired by the observation that the reciprocal $1/\delta$ of the forgetting rate is (apart from constant factors) the information half-life. Note here that compared to~\cite{DoerrDY16PPSN}, we have changed the meaning of $\varepsilon$ and $\delta$, to be in line with the classical literature in machine learning, where the algorithm from~\cite{DoerrDY16PPSN} would be classified as an \emph{$\varepsilon$-greedy} selection scheme (meaning that with probability $\eps$ a random choice is taken and otherwise a greedy choice).

The main theoretical result in~\cite{DoerrDY16PPSN} is a proof that, for suitably selected hyper-parameters $\varepsilon$ and $\delta$, this algorithm essentially always uses the best possible mutation strength when run on \onemax. More precisely, it is shown that in all but a $o(1)$ fraction of the iterations the selected parameter value achieves an expected progress that differs from the best possible one by at most some lower order term. Consequently, the algorithm has the same optimization time (apart from a $o(n)$ additive lower order term) and the same asymptotic 13\% superiority in the fixed budget perspective as the fastest algorithm which can be obtained from these mutation strengths, which again comes arbitrary close (by taking $k$ large) to performance of the hand-crafted mutation strength schedule presented in Theorem~\ref{thm:SADoerrDY16}.

\begin{theorem}[Theorems~1 and~2 in~\cite{DoerrDY16PPSN}]
\label{thm:SADoerrDY16PPSN}
  Let $T(r_{\max})$ be the minimal expected running time that any randomized local search algorithm using a fitness-dependent mutation strength of at most $r_{\max}$ can achieve on \onemax. Then the expected running time $T$ of the $\varepsilon$-greedy RLS variant from~\cite{DoerrDY16PPSN} with hyper-parameters $\epsilon = n^{-0.01}$, $\delta = n^{-0.99}$, and $k=r_{\max}$ is $T(r_{\max}) + o(n)$.
	
	
	In the fixed-budget perspective, the following holds. Let $x^{(t)}_{\varepsilon}$ be the best solution that the $\varepsilon$-greedy RLS variant with this parameter setting has identified within the first $t$ iterations. Similarly, let $x^{(t)}_{\RLS}$ be the best solution that the classic RLS using $1$-bit flips only has found within the first $t$ iterations. For $t \ge 0.2675 n$ the expected Hamming distances to the optimum $z$ satisfy \[E[H(x^{(t)}_{\varepsilon},z)] \le (1+o(1)) \, 0.872 \, E[H(x^{(t)}_{\RLS},z)].\] 
\end{theorem}

The hyper-parameters in this result were taken as one example where this algorithm shows a superior performance. As noted in~\cite{DoerrDY16PPSN}, the particular choice of these parameters is not overly critical. Clearly, $\eps$ has to be $o(1/\log n)$ to ensure that at most $o(n)$ iterations are performed with a sub-optimal mutation strength. Likewise, $\delta$ has to be $\omega(1/n)$ to ensure that information learned $\Omega(n)$ iterations ago (and thus at a time when the velocities could be substantially different) has no significant influence on the current decision.

In addition to this theoretical result, \cite{DoerrDY16PPSN} also presents empirical results for the \leadingones and the minimum spanning tree (MST) problem. These experimental works suggest that, for suitably chosen hyper-parameters $\epsilon$, $\delta$, and $k$, the average optimization time of the $\varepsilon$-greedy RLS variant can be significantly smaller than that of the \oea. It even outperforms, empirically, RLS on \leadingones, and the RLS variant that always flips either one or two random bits in the current-best solution on the MST problem.


\section{Self-Adaptation: Endogenous Parameter Control}\label{sec:SAselfadaptive}

As we have seen in the previous sections, an elegant way to overcome the difficulty of finding the right parameters of an evolutionary algorithm and to cope with the fact that the optimal parameter values may change during a run of the algorithm is to let the algorithm optimize the parameters \emph{on the fly}. However, formally speaking, this is an even more complicated task, because we now have to design a suitable parameter setting mechanism. While a number of natural heuristics like the $1/5$-th rule have proven to be effective in certain cases, it would be even more elegant to not add an \emph{exogenous} parameter control mechanism onto the algorithm, but to rather integrate the parameter control mechanism into the evolutionary process, that is, to attach the parameter value to the individual (consequently, there is no global parameter value, but each individual carries its own parameter value), to modify it via (extended) variation operators, and to use the fitness-based selection mechanism of the algorithm to ensure that good parameter values become dominant in the population. 

This \emph{self-adaptation} of the parameter values has two main advantages. 
\begin{itemize}
\item It is generic, that is, the adaptation mechanism is provided by the algorithm and only the representation of the parameter in the individual and the extension of the variation operators has to be provided by the user. 
\item It allows to re-use existing algorithms and existing code. 
\end{itemize}
Despite these advantages, self-adaptation is not used a lot in discrete evolutionary optimization (unlike in continuous optimization), and consequently, there is also little theoretical work on this topic.

Self-adaptation for discrete evolutionary computation was proposed in the seminal paper~\cite{Back92} by B\"ack, which also contains a mathematical convergence proof for the mutation rate (in the particular setting proposed there). 
Apart from this result, only two works on running time analysis for self-adapting parameter choices appeared so far. Since these results, as the paper by B\"ack, are concerned with self-adaptive mutation rates, we discuss self-adaptation only for mutation rates in the following and note that other parameters could be optimized via self-adaptation in a similar way. 

\subsection{Implementing Self-Adaptive Mutation Rates}

To use self-adaptation for the mutation rate, the individuals (which are usually possible solution candidates) have to be extended to also contain ``their'' mutation rate. In the purest possible form, as done by B\"ack~\cite{Back92}, this is implemented via appending additional bits to the bit-string which represents the solution candidate. These additional bits encode in a suitable manner the mutation rate. This pure form has the advantage that any standard variation operator can be used directly on the extended individuals. The down-side of this approach is that non-binary data is artificially treated like binary decision variables. 

It has been argued, e.g., in~\cite{DoerrDK18}, that it can be preferable to encode non-binary data in their original form and to modify it via data-specific variation operators. In the context of self-adaptation, the mutation rate has been encoded as floating point number in $]0,1[$ in~\cite{BackS96,KruisselbrinkLREB11}, which is mutated according to a log-normal distribution. In the recent theoretical works~\cite{DangL16} and~\cite{DoerrWY18}, only a discrete set of possible mutation rates was allowed. In~\cite{DoerrWY18}, the mutation rates $r/n$ with $r \in [1..n/2]$ being a power of two were used. As mutation, the rate $r/n$ was replaced by a random choice between $(r/2)/n$ and $(2r)/n$.

With either representation of the mutation rate, the extended mutation operator (acting on the extended individuals) will always be such that first the encoded mutation rate is mutated and then the core individual is mutated with this new rate. This is necessary for the subsequent selection step to see an influence of the new mutation rate and thus, hopefully, prefer individuals with a more profitable mutation rate.

Finally, when designing a self-adaptive parameter optimization scheme one may want to prefer non-elitist algorithms. An elitist algorithm carries the risk of getting stuck with individuals that have a high fitness, but a very unprofitable mutation rate. In this situation, progress can only be made when the mutation of the mutation rate in one iteration changes the rate to a value that admits an improvement. In other words, it is not possible to change the rate in several iterations if no improvement is made. 

\subsection{Theory for Self-Adaptive Mutation Rates}\label{sec:SAendogeneousTheory}

In the first work analyzing self-adaptation through the running time analysis paradigm, Dang and Lehre~\cite{DangL16} regard the following setting. They use a simple non-elitist algorithm which in each iteration generates from a population of $\lambda$ individuals a new population of again $\lambda$ individuals. This is done by $\lambda$ times independently selecting an (extended) parent individual from the current population, mutating it via the (extended) mutation operator, and adding it to the new population. For the mutation rate, Dang and Lehre assume that there is only a finite set $\mathcal{M}$ of pre-specified rates (for most results they take $|\mathcal{M}|=2$). The extended mutation operator first with probability~$p$, which is a global parameter of the algorithm, replaces the current rate of the individual by a random different one, then it mutates the core individual via standard bit mutation with the new rate. For the selection operator, a wide range of choices are subsumed in this work, since the results are phrased in terms of a parameter of the selection operator, namely the reproductive rate. A selection operator (possibly depending on a fitness function $f$) has reproductive rate $\alpha$ if for all populations $P$ and each individual $x$ of the population, the expected number of times $x$ was chosen in $\lambda$ independent applications of the selection operator, is at most $\alpha$. For example, selecting always a best individual from the population leads to $\alpha = \lambda$, whereas a uniform random selection gives $\alpha = 1$.

For this setting, the following results are shown. If a mutation rate $p_1$ satisfies $p_1 \ge (\ln \alpha + \delta)/n$ for some constant $\delta$, then the algorithm always using the rate $p_1$ (equivalent to the case that $\mathcal{M} = \{p_1\}$) and using random initialization needs with high probability an at least exponential time to reach the optimum of any pseudo-Boolean function with unique optimum (this is Theorem~2 of~\cite{DangL16} in the special case of $|\mathcal{M}|=1$). 

If two rates are used, that is, $\mathcal{M} = \{p_1,p_2\}$, and the mutation operator chooses the current rate of the individual uniformly at random, then even if only one of the rates satisfies the dangerous condition $p_i \ge (\ln \alpha + \delta)/n$, the above problem can remain: If $p_1 \ge (\ln \alpha - \ln(1+\delta_1))/n$, $p_2 \ge (\ln \alpha - \ln(1 - \delta_2))/n$, and $\delta_1 / (\delta_1+\delta_2) \le \tfrac 12 - \eps$ for constants $\delta_1, \delta_2, \eps > 0$, then again an at least exponential running time results with high probability (Theorem~4). This result again applies to any pseudo-Boolean function $f : \{0,1\}^n \to \R$ having a unique optimum. 

The latter of these two results shows that randomly mixing a good and a bad operator can be essentially as bad as using the bad operator alone. This is not overly surprising, but points out the contrast with the following result for a self-adaptive choice of the mutation rate. For a suitable 
example function~$f$ it is proven that the algorithm with a suitably initialized population, with tournament selection with tournament size $2$, with population size $\lambda \ge c \ln(n)$, and with a self-adaptive choice between the two mutation rates $p_1 \ge \ln(3)$ and $p_2 = \ln(3/2) - \eps$ finds the optimum of~$f$ in a polynomial running time, whereas either of these two rates alone or randomly mixing between them leads to an at least exponential running time with high probability. 

As for almost all such examples, also this one is slightly artificial and needs quite some assumptions, for example, that all $\lambda$ individuals are initialized with the unique local optimum. Nevertheless, this result demonstrates that self-adaptation can outperform static parameter choices and random mixing. The reason for this is that, as the proofs reveal, the self-adaptation is able to find in relatively short time the mutation rate which is most profitable (as opposed to fixed parameter choices) and to remember it (as opposed to random mixing). 

Very recently, a less artificial example for the use of self-adaptation was presented in~\cite{DoerrWY18}. There it was shown that the \ocl with a self-adaptive choice of the mutation rate can achieve an asymptotically identical performance as the self-adjusting \lea presented in~\cite{DoerrGWY17} (see also~\ref{sec:SADoerrGWY17}). In the self-adaptive setting of~\cite{DoerrWY18}, the extended individuals store their mutation rate, which is $r/n$ for an integer $r \in [32..n/64]$. The extended mutation operator first changes $r$ to $r/32$ or $32r$ (uniform random choice) and then performs standard-bit mutation with the new mutation rate $r/n$. One of the offspring with maximum fitness is selected as new parent individual. In case of ties, individuals with  smaller rate are preferred, which creates a small extra drift towards the usually recommended rates of order $\Theta(1/n)$. It is shown that when $\lambda \ge (\ln n)^{1+\eps}$, then this algorithm finds the optimum of the \onemax function in an expected number of $O(n / \log \lambda + (n \log n)/\lambda)$ iterations, which is the asymptotically best possible running time for $\lambda$-parallel algorithms (cf.\ Theorem~\ref{thm:SABadkobehLS14} cited from~\cite{BadkobehLS14}).


\section{Hyper-Heuristics}\label{sec:SAhyper}

Hyper-heuristics are search or optimization heuristics which during the run of the algorithm choose in a possibly adaptive manner which low-level heuristics to use. Since in some situations hyper-heuristics can closely resemble an adaptive parameter choice, we describe in this section what is known about such hyper-heuristics.

\subsection{Brief Introduction to Hyper-Heuristics}

Hyper-heuristics either choose from a pre-specified set of low-level heuristics (\emph{selection hyper-heuristics}) or try to generate low-level heuristics from existing components (\emph{generation hyper-heuristics}). There is a considerable amount of applied research on generation hyper-heuristics, e.g., for scheduling problems, packing problems, satisfiability, and the traveling salesman problem. However, since there appears to be no theoretical work on generation hyper-heuristics and since, naturally, generation hyper-heuristics are substantially different from parameter control mechanisms, we do not further detail this sub-area and refer, as for all other topics incompletely covered here, to the recent survey~\cite{BurkeGHKOOQ13}.

As true in general for optimization heuristics, hyper-heuristics can also be divided into \emph{construction hyper-heuristics} and \emph{perturbation hyper-heuristics}. The former try to construct a solution from partial solutions. This has led to interesting results, e.g., in production scheduling, educational timetabeling, or vehicle routing. Since constructing a solution from partial solutions necessarily is a rather problem-specific approach, it is not surprising that general theoretical results for this sub-area do not yet exist. 

In contrast, perturbation hyper-heuristics work, in a similar manner as classic evolutionary algorithms, with complete solution candidates, which are randomly modified in the hope of gaining superior solutions. Perturbation selection hyper-heuristics found numerous applications, among others, in various scheduling contexts. The most common form of perturbative selection hyper-heuristics are \emph{single-point searches}, which in a fashion analoguous to \oea{}s and \lea{}s repeat creating one or more offspring from a single parent and selecting the next parent from these offspring and the previous parent. For such selection hyper-heuristics, some general mechanisms how to choose the low-level heuristic creating the offspring were proposed, see Section~\ref{sec:SAHHadv}. 

As said above, selection hyper-heuristics are methods that select, during the run of the algorithm, which one out of several pre-specified simpler algorithmic building blocks to use. When the different pre-specified choices are essentially identical apart from an internal parameter, then this selection hyper-heuristic could equally well be interpreted as a dynamic choice of the internal parameter. For example, when only the two mutation operators are available that flip exactly one or exactly two bits, then a selection hyper-heuristic choosing between them could also be interpreted as the \emph{randomized local search} heuristic using a dynamic choice of the number of bits it flips. Conversely, some of the works described previously 
could equally well be described in the language of simple selection hyper-heuristics. In this text, we follow the language used by the original authors and do not aim at drawing a line between the different fields.

We now describe the main theoretical works that appeared in the hyper-heuristics community as long as they resemble dynamic parameter control mechanisms, the main topic of this chapter. 

\subsection{Random Mixing of Low-Level Heuristics}

\subsubsection{Markov Chain Analyses}

The first theoretical study on selection hyper-heuristics was conducted by He, He, and Dong~\cite{HeHD12}. They regard the variant of the classic \oea which in each iteration selects a mutation operator from a finite set of operators according to a fixed probability distribution. In the hyper-heuristics language, this is a single-point selection heuristics using a \emph{mixed strategy}. He et al.\ show that the asymptotic convergence rate and the asymptotic hitting time resulting from any mixed strategy are not worse than those of resulting from exclusively using the worst of the given operators. 

Some care is necessary when interpreting this result. The asymptotic hitting time as defined in~\cite{HeHD12} is not the asymptotic order of magnitude of the classic hitting time (number of iterations until the optimum was generated), but is the spectral radius $\rho(N)$ of the fundamental matrix $N = (I-T)^{-1}$ of the Markov chain describing the parent individual in a run of this single-point heuristic, where $I$ is the identity matrix and $T$ is the transition matrix restricted to the non-optimal search points. This asymptotic hitting time is only loosely related to the classic hitting time. Denoting by $T_x$ the classic hitting time of this Markov chain (usually called optimization time of the EA) when started in the state $x$, then only the week relation 
\[E_{\min} := \min\{E[T_x] \mid x \in S_{\nonopt}\} \le \rho(N) \le \max\{E[T_x] \mid x \in S_{\nonopt}\} =: E_{\max}\] is known, where $S_{\nonopt}$ is the set of all non-optimal search points. Consequently, the asymptotic hitting time $\rho(N)$ only provides a lower bound for the worst-case expected hitting time $E_{\max}$. Note that the best-case expected hitting time $E_{\min}$ often is very small as witnessed by search points $x$ that are very close to the optimum. Consequently, the lower bound for the worst-case hitting time given by $\rho(N)$ can be relatively weak. Nothing is known how the asymptotic hitting time is related to the running time starting from a random search point, which is the usual performance measure. For these reasons, it is not clear how to translate the result of~\cite{HeHD12} into the classic running time analysis language.

\subsubsection{Running Time Analysis of Mixed Strategies}

The first to conduct a running time analysis for selection hyper-heuristics in the classic methodology were Lehre and \"Oczan~\cite{LehreO13}. In~\cite[Theorem~3]{LehreO13}, it is stated that the \oea\footnote{We note that some authors prefer to call the algorithm used in~\cite{LehreO13} a variant of \emph{randomized local search} rather than an evolutionary algorithm since it only creates offspring in a bounded distance from the parent.} using the mixed strategy of choosing in each iteration the mutation operator randomly between the 1-bit flip operator (with probability $p$) and the $2$-bit flip operator (with probability $1-p$) optimizes the \onemax function in an expected time of at most
\begin{equation}\label{eq:SAfalse}
E[T] \le \min\left\{\frac np (1 + \ln(n)), \frac{n^2}{1-p} \left(1 - \frac 1n\right)\right\} \le 
\begin{cases}
\frac np  (1+\ln n), &\text{ if $p > \frac{1+\ln n}{n+\ln n}$,}\\
\frac{1}{1-p} \, n^2 , &\text{ else.}
\end{cases}
\end{equation}
It appears to us that this result is not absolutely correct, since, e.g., in the case $p=0$ the expected optimization time is clearly infinite: If the random initial search point has an odd Hamming distance from the optimum, then the optimum cannot be reached only via $2$-bit flips. For similar reasons, the expected running time has to be larger than in~\eqref{eq:SAfalse} for very small values of $p$. We therefore prove the following result.

\begin{theorem}\label{thm:SAhyom}
  Consider the \oea with the mixed mutation strategy of flipping a single random bit with probability $p$ and flipping two (different) random bits with probability $1-p$. Let $T$ be the running time (number of iterations) of this algorithm on the \onemax benchmark function. If $p > 0$, then 
  \[E[T] \le  
  \begin{cases}
  \frac np + n^2, & \text{ if $p \le \frac 1n$}\\
  \frac np \left(\ln(np) + 1 + \frac{\ln(np)}{np-1}\right) , & \text{ if $p > \frac 1n$}
  \end{cases}.\] 
  If $p=0$, then with probability $\frac 12$ the algorithm never finds the optimum (and thus the expected running time $E[T]$ is infinite).
\end{theorem}

\begin{proof}
  For the case $p=0$, we note that with probability exactly $\frac 12$ the random initial search point has an odd Hamming distance from the optimum.\footnote{This well-known fact follows from the beautiful argument $0 = (1-1)^n = \sum_{i=0}^n 1^i (-1)^{n-i} \binom{n}{i}$.} Since $2$-bit flips change the Hamming distance by $-2$, $0$, or $+2$, the algorithm can never reach the optimum in this case. 
  
  Hence let us assume $p>0$ for the remainder of this proof. When the current search point of the \oea has a Hamming distance of $d \ge 1$ from the optimum, then the probability $p_d$ that one iteration ends with a better search point is \[p_d = p \, \frac dn + (1-p) \, \frac{d(d-1)}{n(n-1)} = \frac{d((1-p)d+np-1)}{n(n-1)}.\] 
%
%

  Using $p_1 = \frac pn$ and $p_d \ge \frac{d(d-1)}{n(n-1)}$ for all $d \ge 2$, the classic fitness level theorem yields
  \begin{align*}
  E[T] & \le \sum_{d=1}^n p_d^{-1} \\
  & \le \frac np + n(n-1)\sum_{d=2}^n \frac{1}{d(d-1)} \\
  & = \frac np + n^2 \left(1-\frac 1n\right)^2 \le \frac np + n^2.
  \end{align*}
  Above we used the equation $\sum_{d=2}^n \frac{1}{d(d-1)} = 1 - \frac 1n$ valid for all $n \in \N$, which can be shown easily by induction.
  
  For $p > \frac 1n$, we also have the estimate 
\begin{align*}
  E[T] & \le \sum_{d=1}^n p_d^{-1} \\
  & = n(n-1) \sum_{d=1}^n \frac{1}{d((1-p)d+np-1)}\\
  & \le n(n-1)\left(\frac{1}{(n-1)p} + \int_{1}^n \frac{1}{d((1-p)d+np-1)} \mathrm{d}d \right)\\
  & = \frac np + n (n-1) \left.\left(-\frac{1}{np-1} \ln\left(\frac{(1-p)d+np-1}{d}\right)\right)\right|_{1}^n\\
  & = \frac np + \frac{n(n-1)}{np-1} \left(\ln((n-1)p) - \ln\left(\frac{n-1}{n}\right)\right)\\
  &\le  \frac np + n^2 \frac{\ln(np)}{np-1} = \frac np \left(1 + \ln(np) \left(1+\frac{1}{np-1}\right)\right) = \frac np \left(\ln(np) + 1 + \frac{\ln(np)}{np-1}\right).
  \end{align*}
  Note that for all $p > \frac 1n$, we have $\ln(np) < np-1$. Hence the bound above is less than $\frac np + n^2$ and thus stronger than the first bound.
\end{proof}

Without giving full details, we remark that better results can be obtained by using variable drift instead of the classic fitness level method. Since a $2$-bit flip giving a fitness improvement automatically improves the fitness by exactly two, we have that the expected fitness gain in one iteration starting with a search point with fitness distance $d$ is 
\[h_d = p \, \frac dn + 2(1-p) \, \frac{d(d-1)}{n(n-1)} = \frac{d(2(1-p)d+np+p-2)}{n(n-1)}.\] 
Now the variable drift theorem for upper bounds on hitting times~(see~\cite{Johannsen10}, note that for processes in $\N_0$ the integration can be replaced by a summation) gives $E[T] \le \sum_{d=1}^n h_d^{-1}$, which can be estimated in a similar fashion as the term $\sum_{d=1}^n p_d^{-1}$ above. What is more interesting than the slightly improved upper bound is that the variable drift theorem for lower bounds~\cite{DoerrFW11} gives a very similar lower bound, namely $E[T] \ge \sum_{d=3}^n h_d^{-1}$; note again that for integer valued processes we can replace the integration with a summation.


The above results show that for the classic benchmark function \onemax mixing the $1$-bit and $2$-bit operators in a random fashion gives no improvement over exclusively using the $1$-bit operator. In light of the precise analysis of the performance of $k$-bit flip operators on \onemax in~\cite{DoerrDY16}, this result is not very surprising. There, it was shown that the expected fitness gain is never maximized by flipping an even number of bits. Also, from a fitness of $(\frac 23 + o(1)) n$ on, the $1$-bit flip operator is the only one maximizing the expected fitness gain. 

\subsubsection{Superiority of Mixed Strategies}

To demonstrate the use of mixing operators, Lehre and \"Oczan~\cite{LehreO13} construct an example function \textsc{GapPath}, which has the property that the \oea mixing $1$-bit and $2$-bit flips when initialized with $x_0 = (0,\dots,0)$ can  optimize \textsc{GapPath} only when both the $1$-bit and the $2$-bit flip mutation operator are chosen with positive probability. Based on this result, several ways to alternate between a low and a high $p$-value are discussed, including a success-based reinforcement approach. While these ideas are shown to give improvements over certain choices of $p$ like $p=\frac 1n$, they do not outperform natural choices like $p=\frac 12$ or $p=1$. 

An example similar to \textsc{GapPath} was used to show that mixing $1$-bit and $2$-bit flip operators can be necessary also in multi-objective optimization~\cite{QianTZ16}.

We note that a more natural example for the need of mixing, without being explicitly stated there, was already regarded by Neumann and Wegener~\cite{NeumannW07} (and a slightly more technical example was given even earlier by Giel and Wegener~\cite{GielW03}). Neumann and Wegener~\cite{NeumannW07} analyze how simple randomized search heuristics solve the minimum spanning tree problem in connected undirected graphs $G = (V,E)$ having $n :=|V|$ vertices and $m:=|E|$ edges with integral edge weights in $[1..w_{\max}]$. They use the natural representation that individuals are sets $S=S(x)$ of edges represented via bit-strings $x \in \{0,1\}^E$. As fitness (to be minimized) of an individual they propose 
\[f(x) = M^2 (C_{x}-1) + M \left(\sum_{e \in S(x)} x_e - (n-1)\right) + \sum_{e \in S(x)} x_e w(e),\]
where $M = n^2 w_{\max}$ and $C_x$ is the number of connected components of the graph $(V,\{e \in E \mid x_e = 1\})$. This fitness function with first priority punishes connected components, then punished the number of edges, and only then prefers solutions with smaller total weight (we do not see that the punishment of edges is necessary, but clearly it does not harm either). Besides the \oea, they analyze the performance of (in their language) a variant of the randomized local search heuristic which in each iteration either (uniform random choice) flips a single random bit or two different random bits. In the hyper-heuristics language, they thus regard the same single-point selection hyper-heuristic with random mixing between the $1$-bit and the $2$-bit flip operator as~\cite{LehreO13} except that they fix the probability $p$ to $\frac 12$. 

Neumann and Wegener show that this algorithm computes a minimum spanning tree in an expected number of $O(m^2 \log(n w_{\max}))$ iterations. It can easily be seen and has been shown in~\cite{ReichelS09} that for this algorithm, the $w_{\max}$ term in the running time bound can be omitted, but we shall not care about this usually small improvement in the following. Neumann and Wegener do not make this explicit, but from their proofs it is clear that any other mixing which uses both operators with constant probability would give the same result. The reason why Neumann and Wegener use both $1$-bit and $2$-bit flips is that, obviously, all spanning trees are local optima of the fitness function. Consequently, using $1$-bit flips only bears the risk of getting stuck in a local optimum forever. The parity argument used in the proof of Theorem~\ref{thm:SAhyom} shows that also when using only the $2$-bit flip operator, the algorithm has a constant probability (of exactly $1/2$) of never reaching an optimum. 

\begin{theorem}[analogous to Theorem~11 of~\cite{NeumannW07}]\label{thm:SAhyperMST}
  Consider the \oea with the mixed strategy of flipping one random bit (with probability $p$) and two different random bits (with probability $1-p$) solving the minimum spanning tree problem in connected undirected graphs having $n$ vertices, $m$ edges, and integral edge weights in $[1..w_{\max}]$.
  \begin{itemize}
  \item If both $p$ and $1-p$ are $\Omega(1)$, then the expected optimization time is $O(m^2 \log(nw_{\max}))$. 
  \item If $p = 0$, then with probability $1/2$ the algorithm never finds any spanning tree.
  \item If $p = 1$ and the input graph is does not have the property that each spanning tree is a minimum spanning tree, then with positive probability the algorithm never finds a minimum spanning tree.
  \end{itemize}
  Consequently, this algorithm solves the minimum spanning tree problem in polynomial expected time if and only if $p \notin \{0,1\}$, that is, if there is a true mixing of the two mutation operators.
\end{theorem}

The works~\cite{GielW03,NeumannW07} show that hyper-heuristics using random mixing of mutation operators could with equal justification just be called evolutionary algorithms using a possibly non-standard mutation operator. With equal justification, one could declare the \oea or the \lea using the  classic standard-bit mutation operator (flipping each bit independently with probability~$\frac 1n$) a single-point selection hyper-heuristic choosing the $k$-bit flip operator with probability exactly $\binom{n}{k} (\frac 1n)^k (1-\frac 1n)^{n-k}$. The same statement (with a different probability distribution) is true when the heavy-tailed mutation operator of~\cite{DoerrLMN17} is used instead of standard-bit mutation.

We end this section with a recent result giving an example where a large number of mixings give asymptotically the same performance. In~\cite{AntipovD18}, the plateau functions $\plateau_k$ is defined by $\plateau_k(x) = \onemax(x)$, if $\onemax(x) \in [0..n-k] \cup \{n\}$, and $\plateau(x) = n-k$ if $\onemax(x) \in [n-k+1..n-1]$. This function thus agrees with the \onemax function except that it has a large plateau of size $N = \sum_{i = 1}^{k} \binom{n}{i} = \frac{n^k}{k!} + o(n^k)$ around the optimum. Consider the \oea randomly mixing the $k$ mutation operators which flip exactly $1, 2, \dots, k$ bits. Let $p_1, p_2, \dots, p_k \in [0,1]$ with $\sum_{i=1}^k p_i = 1$ be the probabilities of selecting the corresponding operators (and view these numbers as constants, that is, not depending on $n$). Assume $p_1 > 0$ to ensure that the algorithm surely converges. Then the expected optimization time is $E[T] = (1+o(1)) N$ regardless of the values of $p_1, \dots, p_k$.

\subsection[Beyond Mixing: Advanced Selection Mechanisms]{Beyond Mixing: Advanced Selection Mechanisms\footnote{Warning: All results described in this section use a different definition of the $2$-bit flip operator, namely the one where independently and uniformly at random $i \in [1..n]$ and $j \in [1..n]$ are chosen and then first the $i$-th bit is flipped and then the $j$-th bit is flipped. Consequently, this operator with probability $1 - 1/n$ indeed flips two random different bit positions. With probability $1/n$, however, we have $i = j$ and thus the two flipping operation cancel, that is, the offspring is identical to the parent. We do not see much reason for the use of this alternative operator. We would suspect (but did not check this rigorously) that all results presented in this section hold as well for the classic $2$-bit flip operator, which flips two randomly chosen different bit position (in other words, returns a random search point with Hamming distance $2$ from the parent).}}\label{sec:SAHHadv}

The first to conduct a theoretical analysis of more sophisticated selection hyper-heuristics were Alanazi and Lehre~\cite{AlanaziL14}. Besides the \emph{simple random} heuristic (choosing a low-level heuristic uniformly at random each time, that is, mixing with uniform distribution), they regard the following classic selection mechanisms.
\begin{itemize}
\item \emph{Random gradient}: take a random low-level heuristic and repeat using it as long as a true fitness improvement is obtained.
\item \emph{Greedy}: in each iteration, use all low-level heuristics in parallel and continue with 
a best search point generated by one of them (or the parent if no offspring is at least as good as the parent).
\item \emph{Permutation}: generate initially a random cyclic order of the low-level heuristics and then use them in this order. This mechanism can be seen as a quasirandom analogue of the \emph{simple random} heuristic (see~\cite{DoerrFW10} for a discussion of the use of quasirandomness in evolutionary computation). Alternatively, this hyper-heuristic can be viewed as a time-dependent parameter control method. In fact, the time-dependent choices of the mutation rate discussed in~\cite{DrosteJW00,JansenW06}, see Section~\ref{sec:SAtime}, can be seen as special cases of this hyper-heuristic.
\end{itemize}
Again for the choice between $1$-bit and $2$-bit flips, they prove upper and lower bounds for the expected optimization time on the \leadingones benchmark function. While the results are relatively tight, the corresponding upper and lower bounds deviate by at most a factor of $6 + o(1)$, the intervals of possible running times intersect. Hence this first running time analysis for these advanced selection mechanisms does not yet give a conclusive picture. 

Given that the probabilities to find a true improvement are very low in this discrete optimization problem, one would expect that the four selection mechanisms all use the two operators in a very balanced manner and thus lead to very similar running times. This is indeed the first set of results in the remarkable work of Lissovoi, Oliveto, and Warwicker~\cite{LissovoiOW17}. Building on the precise analysis method of~\cite{BottcherDN10} instead of the fitness level method, they show that the expected running time for all four selection mechanisms is $\frac 12 \ln(3) \, n^2 + o(n^2) \approx 0.549 n^2$.  Consequently, the more complex heuristics do not give a measurable performance gain over a simple randomized selection of the operator, and all are worse than just using $1$-bit flips, which is known to give an expected running time of precisely $0.5n^2$.
\begin{theorem}[Theorem~4.2 and Corollary~4.3 in~\cite{LissovoiOW17}]\label{thm:SALOW17abc}
  The \oea using one of the selection mechanisms \emph{simple random},  \emph{random gradient}, \emph{greedy}, or \emph{permutation} to choose between the $1$-bit or the $2$-bit flip operator optimizes the \leadingones function in an expected number of $\frac 12 \ln(3) \, n^2 +o(n^2) \approx 0.549 n^2$ iterations. 
\end{theorem}

Lissovoi et al.~\cite{LissovoiOW17} build on this strong result by proposing to use a slower adaptation (a similar idea can be found already in~\cite{AlanaziL14}, there however in a very problem-specific manner and only with preliminary experimental results). For the \emph{random gradient} method, they propose to switch the low-level heuristic only after a phase of $\tau$ iterations. More precisely, the current low-level heuristic is used for up to $\tau$ iterations. If an improvement is found, immediately another phase with this operator starts. If a phase of $\tau$ iterations does not see a fitness improvement, then a new phase is started with a random operator. 

For this \emph{generalized random gradient} mechanism with a phase length of $\tau=cn$ for a constant~$c$, they show (still for the \leadingones problem and the $1$-bit and $2$-bit mutation operators) an expected running time of $g(c) n^2 + o(n^2)$, where $g(c)$ is a constant depending on $c$ only that tends to $\frac{\ln(2)+1}{4} \approx 0.423$ when $c$ is tending to infinity. Consequently, this new hyper-heuristic outperforms the previously investigated ones when $c$ is large enough. For $c$ tending to infinity, its performance approaches the best-possible performance that can be obtained from the two mutation operators, which is, as also shown in~\cite{LissovoiOW17}, $\frac{\ln(2)+1}{4} n^2 +o(n^2)$. The following variant of this result appeared in the preprint~\cite{LissovoiOW18}.

\begin{theorem}[Theorem~7 and Corollary~15 in~\cite{LissovoiOW18}]\label{thm:SALOW18}
  Consider the \oea using the \emph{generalized random gradient} selection heuristic with phase length $\tau \in \omega(n)$ and $\tau \le c n \ln(n)$, $c < \frac 12$, to choose between the $1$-bit and $2$-bit flip operator. Then this algorithm optimizes the \leadingones function in an expected number of $\frac 14 (\ln(2)+1) n^2 + o(n^2) \approx 0.423 n^2$ iterations. This is, apart from lower order terms, the best running time which can be achieved with these two mutation operators.
\end{theorem}

A similarly generalized variant of the \emph{greedy} selection hyper-heuristics is also found to improve over the classic selection heuristics, but appears not to give the same good results as the \emph{generalized random gradient} method.

The \emph{generalized random gradient} heuristic was further extended in~\cite{DoerrLOW18}. There an operator was defined as successful (which leads to another phase using this operator) if it leads to $\sigma$ improvements in a phase of at most $\tau$ iterations. Hence in this language, the previous \emph{generalized random gradient} heuristic uses $\sigma=1$. By using a larger value of $\sigma$, the algorithm is able to take more robust decisions on what is a success. This was used in~\cite{DoerrLOW18} to determine the phase length $\tau$ in a self-adjusting manner. While the previous work~\cite{LissovoiOW17} does not state this explicitly, the choice of $\tau$ is crucial. A $\tau$-value of smaller asymptotic order than $\Theta(n)$ leads to typically no improvement within a phase and thus reverts the algorithms to the simple random heuristic. A $\tau$-value of more than $c n \ln(n)$, where $c$ is a suitable constant, results in that both operators are successful in most parts of the search space. Consequently, the algorithm sticks to the first choice for a large majority of the optimization process and thus does not profit from the availability of both operators. 

Since the choice of $\tau$ is that critical, a mechanism successfully adjusting it to the right value is  desirable. In~\cite{DoerrLOW18} it is shown that by choosing $\sigma \in \Omega(\log^4 n) \cap o(\sqrt{n / \log n})$---note that this is a quite wide range---the value of $\tau$ can be easily adjusted on the fly via a multiplicative update rule. This gives again the asymptotically optimal running time of Theorem~\ref{thm:SALOW18}.

\section{Conclusion and Outlook}\label{sec:SAconclusions}

The recent years saw a significant increase in our understanding of parameter control. The results stemming from the theory community indicate that success-based rules can easily lead to good parameter settings. These rules are easy to find due to their intuitive hyper-parameters: If we conduct the update by multiplying the parameter with $F$ in case of a success and with $F^{1/(\sigma-1)}$ in case of failure, then $F$ controls the speed of adaptation and $1 / \sigma$ is the intended rate of successes (e.g., $\sigma = 5$ in the case of the classic $1/5$-th rule). It is also easy to observe if such an update rule works as desired or not: If the aimed-at rate of successes can not be obtained, then imbalance in the updates leads to an exponential growth or shrinking of the parameter value. Therefore, we currently see no reason to not try such a multiplicative update rule in a situation where one expects a monotonic relation between a parameter and the success of an iteration. 

The increased power of learning-based approaches (being able to gather and exploit information obtained over many iterations) or self-adaptation suggest that one should not ignore these, however, our current understanding here is more limited. Indeed, we feel that making these directions more usable is among the following \textbf{open problems} we want to mention.

\begin{itemize}
	\item \textbf{Theory for learning-inspired parameter control mechanisms.} While there has been considerable momentum for empirical works on learning-inspired parameter control mechanisms~\cite{Thierens05,DaCostaGECCO08,FialhoCSS08,FialhoCSS10,LiFKZ14}, these mechanisms still lack a solid mathematical foundation. The only result that we are aware of in this context is the (almost) optimality of the $\varepsilon$-greedy RLS variant presented in~\cite{DoerrDY16PPSN}, cf. Section~\ref{sec:SAlearning} above. In addition to its intrinsic motivation, this research direction will most probably result in a better reconciliation of research activities in optimization and machine learning, where many of the empirically tested techniques stem from. 
  \item \textbf{Understanding self-adaptation.} While self-adaptation is massively used in continuous evolutionary optimization, it only plays a marginal role in discrete optimization. The general hope that the inclusion of the adaptive process into the main evolutionary algorithm easily automates the on-the-fly control of parameters has not yet come true. The two, very recent, theoretical works on this topic suggest, however, that self-adaptation can work. Therefore, extending these first works towards a more profound understanding how to use self-adaptation in discrete evolutionary optimization seems to be both a profitable and a feasible endeavor. 
	\item \textbf{Controlling more than one parameter.} As indicated in Section~\ref{sec:SAintro}, even for static parameter settings we do not have many examples of running time bounds that depend on two or more parameters, with the exceptions of a bound for the $(1+\lambda)$~EA with mutation rate $p=c/n$ proven in~\cite{GiessenW17Algorithmica}, a tight running time analysis for the $(\mu+\lambda)$~EA~\cite{AntipovDFH18}, and the 3-dimensional analysis of the \ga presented in~\cite{DoerrD18ga}. For non-static parameter choices the complexity of the analysis increases considerably, as the parameters often interact in a difficult to analyze manner. We are not aware of any theoretical result addressing the control of two or more parameters. According to~\cite{KarafotiasHE15} also the empirical works focus mostly on controlling a single parameter, while for the simultaneous adaptation of two or more parameters only few mechanisms have been tested. 
\end{itemize}

\subsubsection*{Acknowledgments} 
We thank Thomas B\"ack for many useful discussions on this topic. We thank Franziska Huth for providing Figure~\ref{fig:SAgalambda}. This work was supported by a public grant as part of the Investissement d'avenir project, reference ANR-11-LABX-0056-LMH, LabEx LMH, in a joint call with the Gaspard Monge Program for optimization, operations research, and their interactions with data sciences.

}
%

\newcommand{\etalchar}[1]{$^{#1}$}

\begin{sidewaystable}														
\begin{center}
{\footnotesize{														
\begin{tabular}{c|c|c|c|l|c|c}														
\textbf{Algorithm}	& 	\textbf{Dynamic Parameter}	& 	\textbf{Control Scheme}	& 	\textbf{Function}	& 	\textbf{Results}	& 	\textbf{Reference}	& 	\textbf{Sec./Thm.}	\\
\hline
\rowcolor[gray]{.9}\multicolumn{7}{l}{\textbf{State-Dependent Parameter Control Schemes}}\\
\hline												
(1+1) EA	& 	p	& 	time-dep.	& 	\textsc{PathToJump}	& 	$O(n^2 \log n)$	& 	\cite{DrosteJW00,JansenW06}	& 	Sec.~\ref{sec:SAtime}	\\
$(\mu+1)$ EA	& 	p	& 	rank-based	& 	\onemax	& 	$O(\mu n \log n)$	& 	\cite{OlivetoLN09}	& 	Thm.~\ref{thm:SAOlivetoLN09}	\\
$(\mu+1)$ EA	& 	p	& 	rank-based	& 	$f:\{0,1\}^n \to \R$	& 	$7 \cdot 3^n$	& 	\cite{OlivetoLN09}	& 	Thm.~\ref{thm:SAOlivetoLN09}	\\
(1+1) EA	& 	p	& 	fitness-dep.	& 	\leadingones	& 	$0.68 n^2$	& 	\cite{BottcherDN10}	& 	Thm.~\ref{thm:SABottcherDN10}	\\
$(1+\lambda)$ EA	& 	p	& 	fitness-dep.	& 	\onemax	& 	$\Theta(n \log n + n \lambda /\log \lambda )$	& 	\cite{BadkobehLS14}	& 	Thm.~\ref{thm:SABadkobehLS14}	\\
RLS	& 	step size $\ell$	& 	fitness-dep.	& 	\onemax	& 	$n \ln(n)-cn\pm o(n)$	& 	\cite{DoerrDY16}	& 	Thm.~\ref{thm:SADoerrDY16}	\\
\ga	& 	$\lambda$	& 	fitness-dep.	& 	\onemax	& 	$\Theta(n)$	& 	\cite{DoerrDE15}	& 	Thm.~\ref{thm:SAgafitness}	\\
\hline
\rowcolor[gray]{.9}\multicolumn{7}{l}{\textbf{Success-Based Parameter Control Schemes}}													\\
\hline
\ga	& 	$\lambda$	& 	$1/5$-th success rule	& 	\onemax	& 	$\Theta(n)$	& 	\cite{DoerrD18ga}	& 	Thm.~\ref{thm:SAga15}	\\
\ga	& 	$\lambda$	& 	$1/5$-th success rule	& 	\textsc{MaxSAT}	& $ O(\max\{n,(n \log n)/(\frac mn)^{4+\eps}\})$	& 	\cite{BuzdalovD17}	& 	Sec.~\ref{sec:SAgaadaptive}	\\
$(1+\lambda)$ EA	& 	$\lambda$	& 	$\{2\lambda,\lambda/2\}$, $\{2\lambda,1\}$	& \multicolumn{2}{c|}{cf.\ Table~\ref{tab:SALassigS11}\hspace*{2cm}}			& 	\cite{LassigS11}	& 	Thm.~\ref{thm:SALassigS11}	\\
$(1+\lambda)$ EA	& 	p	& 	2-rate $\{2r,r/2\}$	& 	\onemax	& 	$\Theta(n \log n + n\lambda/\log \lambda)$	& 	\cite{DoerrGWY17}	& 	Thm.~\ref{thm:SADoerrGWY17}	\\
RLS	& 	pos.-dep. step size $\ell_i$	& 	${A\ell_i,b\ell_i}$	& 	$r$-ary \onemax	& 	$\Theta(n(\log n + \log r))$	& 	\cite{DoerrDK18}	& 	Thm.~\ref{thm:SADoerrDK17algo}	\\
parallel (1+1) EA	& 	migration interval $\tau$	& 	$\{2\tau,\tau/2\}$, $\{2\tau,1\}$	& \multicolumn{2}{c|}{cf.\ Table~\ref{tab:SALassigS11}\hspace*{2cm}}	& 	\cite{MambriniS15}	& 	Sec.~\ref{sec:SAMambriniS15}	\\
\hline
\rowcolor[gray]{.9}\multicolumn{7}{l}{\textbf{Learning-Inspired Parameter Control Schemes}}													\\
\hline
RLS	& 	step size $\ell$	& 	$\varepsilon$-greedy 	& 	\onemax	& 	$n \ln(n)-cn\pm o(n)$	& 	\cite{DoerrDY16PPSN}	& 	Thm.~\ref{thm:SADoerrDY16PPSN}	\\
\hline
\rowcolor[gray]{.9}\multicolumn{7}{l}{\textbf{Endogeneous (Self-Adaptive) Parameter Control Schemes}}													\\
\hline
$(\lambda,\lambda)$ EA	&	p	&	self-adaptation	&	\multicolumn{2}{c|}{artificial example showing that self-a. can work} &	\cite{DangL16}	&	Sec.~\ref{sec:SAendogeneousTheory}	\\
$(1,\lambda)$ EA	&	p	&	self-adaptation	&	\onemax	&	$\Theta(n \log n + n\lambda/\log \lambda)$ &	\cite{DoerrWY18}	&	Sec.~\ref{sec:SAendogeneousTheory}	\\
\hline
\rowcolor[gray]{.9}\multicolumn{7}{l}{\textbf{Hyper-Heuristics}}													\\
\hline
RLS	&	step size $\ell$	&	random mixing 
&	\onemax	&	$\frac np (\ln(np)+O(1))$ for $p>\frac 1n$	&	\cite{LehreO13}, Thm.~\ref{thm:SAhyom}	&	Thm.~\ref{thm:SAhyom}	\\
RLS	&	step size $\ell$	&	random mixing 
&	MST	&	$O(m^2 \log(nw_{\max}))$	&	\cite{NeumannW07}	&	Thm.~\ref{thm:SAhyperMST}	\\
RLS	&	step size $\ell$	&	classic schemes
&	\leadingones	&	$\approx 0.549 n^2$	&	\cite{LissovoiOW17}	&	Thm.~\ref{thm:SALOW17abc}	\\
RLS	&	step size $\ell$	&	gen. rand. grad.
&	\leadingones	&	$\approx 0.423 n^2$	&	\cite{LissovoiOW17}	&	Thm.~\ref{thm:SALOW18}	\\
RLS	&	\begin{tabular}{@{}c@{}}step size $\ell$, \\ window $\tau$\end{tabular} 	&	\begin{tabular}{@{}c@{}}gen. rand. grad., \\ (1+o(1)) success rule\end{tabular}	&	\leadingones	&	$\approx 0.423 n^2$	&	\cite{DoerrLOW18}	&	Sec.~\ref{sec:SAHHadv}	\\	
\hline														
\end{tabular}
}}
\caption{Summary of selected theoretical running time bounds, sorted by parameter control scheme. We report the expected number of function evaluations needed to identify an optimal solution, not the number of generations. For ease of comparison, we refer to the algorithms regarded in Section~\ref{sec:SAhyper} as RLS variants, and not as (1+1) EAs. Abbreviations: gen. rand. grad. = generalized random gradient, self-a. = self-adaptation}
\label{tab:SAsummaryResults}														
\end{center}														
\end{sidewaystable}

\end{document}